%% file: CameraReady2027.tex
\documentclass[letterpaper]{article} 
\usepackage{aaai2027}  
\usepackage[hyphens]{url}  
\usepackage{graphicx} 
\usepackage{natbib}  
\usepackage{caption} 
\usepackage{algorithm}
\usepackage{algorithmic}

\usepackage{newfloat}
\usepackage{listings}
\DeclareCaptionStyle{ruled}{labelfont=normalfont,labelsep=colon,strut=off} 
\floatstyle{ruled}
\newfloat{listing}{tb}{lst}{}
\floatname{listing}{Listing}

\usepackage{booktabs}

\usepackage[most]{tcolorbox}
\usepackage{enumitem}
\usepackage{xcolor}

\usepackage{makecell}
\usepackage{dcolumn}

\usepackage{array}

\usepackage{booktabs}
\usepackage{multirow}
\usepackage{subcaption}
\usepackage{amsmath}

\usepackage{algorithm}
\usepackage{algorithmic}

\usepackage{newfloat}
\usepackage{listings}
\DeclareCaptionStyle{ruled}{labelfont=normalfont,labelsep=colon,strut=off} 
\floatstyle{ruled}
\newfloat{listing}{tb}{lst}{}
\floatname{listing}{Listing}

\usepackage{tabularx}
\usepackage{amssymb}
\usepackage{amsthm}

\setlist[itemize]{leftmargin=1.15em,itemsep=1.2pt,topsep=1.2pt}
\newcommand{\tokend}{\texttt{\textless/end/\textgreater}}
\newcommand{\tokenok}{\texttt{\textless ok\textgreater}}

\newtheorem{theorem}{Theorem}
\newtheorem{lemma}[theorem]{Lemma}

\definecolor{AFlowCardFrame}{HTML}{6F8793}
\definecolor{AFlowCardTitle}{HTML}{EAF0F3}
\definecolor{AFlowCardBack}{HTML}{FFFFFF}

\newtcolorbox{aflowcard}[1]{
  enhanced,
  width=\linewidth,
  title={#1},
  colframe=AFlowCardFrame,
  colback=AFlowCardBack,
  colbacktitle=AFlowCardTitle,
  coltitle=black,
  boxrule=0.45pt,
  titlerule=0pt,
  arc=1mm,
  left=1.5mm,
  right=1.5mm,
  top=1.3mm,
  bottom=1.3mm,
  before skip=0pt,
  after skip=0pt,
  fonttitle=\bfseries,
  fontupper=\footnotesize
}

\newcommand{\aflowrole}[1]{%
  \par\noindent\textbf{#1}\enspace
}

\newcommand{\aflowsep}{%
  \par\vspace{0.35em}%
  {\color{black!18}\hrule height 0.3pt}%
  \vspace{0.35em}%
}

\title{Affective Flow Language Model for Emotional Support Conversation}
\author{
    Chenghui Zou\textsuperscript{\rm 1},
    Ning Wang\textsuperscript{\rm 1}\corresponding,
    Tiesunlong Shen\textsuperscript{\rm 2}\corresponding,
    Luwei Xiao\textsuperscript{\rm 2},
    Chuan Ma\textsuperscript{\rm 1},
    Xiangpeng Li\textsuperscript{\rm 1},
    Rui Mao\textsuperscript{\rm 3},
    Erik Cambria\textsuperscript{\rm 3}
}

\affiliations{
    \textsuperscript{\rm 1}College of Computer Science, Chongqing University\\
    \textsuperscript{\rm 2}National University of Singapore\\
    \textsuperscript{\rm 3}Nanyang Technological University\\
    zouchenghui@stu.cqu.edu.cn,
    nwang5@cqu.edu.cn,
    tiesunlong@nus.edu.sg,
    louisshaw008@gmail.com,\\
    chuan.ma@cqu.edu.cn,
    xiangpengli.cs@gmail.com,
    \{rui.mao, cambria\}@ntu.edu.sg
}

\begin{document}

\nocopyright
\maketitle

\begin{abstract}
Large language models  (LLMs)  have advanced emotional support conversation, but existing alignment methods rely mainly on sparse preferences at the response level or outcomes at the dialogue level, providing limited supervision for sequential strategy decisions in multi-turn interactions. This raises a key question: how can detailed process signals be derived from overall dialogue outcomes to guide the gradual adaptation of support strategies? We propose the Affective Flow Language Model (AFlow), which models multi-turn emotional support as an affective utility flow evolving along dialogue trajectories. AFlow searches diverse support trajectories and estimates the utility of intermediate dialogue states and candidate strategies. It further introduces Affective Flow Preference Optimization (AFPO), which uses a flow-balance objective defined over dialogue subpaths to propagate downstream preference signals to intermediate states and learn strategy transitions consistent with support outcomes over the full dialogue. AFlow introduces flow-balance learning into multi-turn affective interaction, providing a process-based approach to dynamic affect modeling and continuous strategy optimization. Experiments on ExTES and ESConv show consistent improvements in strategy alignment, response diversity, and generation quality across different model environments and evaluation settings.  Our code is available at \textcolor{blue}{\url{https://github.com/chz2025/AffectiveFlow}}. 
\end{abstract}

\section{Introduction}
\label{sec:intro}

Emotional Support Conversation (ESC) has emerged as an important research area in natural language processing, aiming to alleviate users' emotional distress through strategy-guided multi-turn interactions~\cite{chen2025socialsim,zhao2024esc}. Traditional ESC systems still face substantial challenges in capturing complex affective dynamics, generating contextually appropriate responses, and maintaining coherent strategy progression across turns~\cite{cheng2022improving}. In recent years, large language models (LLMs) have been increasingly adopted for emotional support because of their stronger semantic understanding and open-domain dialogue capabilities~\cite{kang2024can,wan2025emodynamix}. However, existing methods still struggle to maintain consistent long-horizon strategy progression and achieve fine-grained alignment with professional psychological support principles, limiting their effectiveness in complex multi-turn interactions~\cite{zhang2024escot,zhou2025t}.

\begin{figure}[t]
    \centering
    \includegraphics[width=1.0\linewidth]{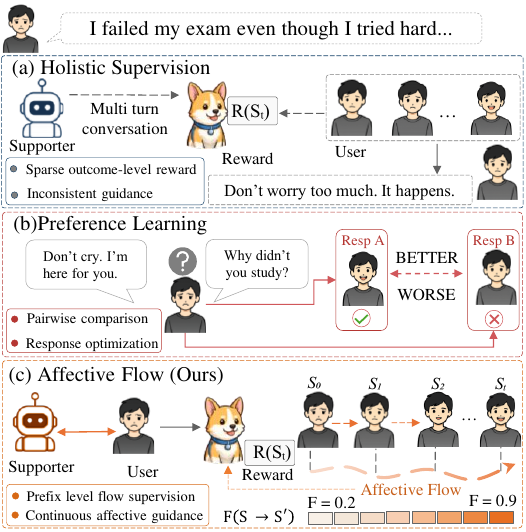}
    \caption{Comparison of emotional support conversation approaches.}
    \label{fig:framework}
\end{figure}

To address these issues, researchers have mainly explored two paradigms: holistic supervision and preference learning. As shown in Figure~\ref{fig:framework}(a), holistic supervision methods typically assess support quality using outcome-level rewards after the completion of a dialogue. Although these methods can improve overall dialogue performance, sparse terminal feedback provides limited guidance for intermediate strategy decisions~\cite{liu2021towards,hong2024interactive}. As shown in Figure~\ref{fig:framework}(b), preference-learning methods generate multiple candidate responses, identify preferred and dispreferred responses through pairwise comparisons, and optimize the model using response-level preference signals~\cite{rafailov2023direct,shen2025reasoning,zhao2025chain}. Although these methods improve single-turn response quality, they remain limited in modeling the cumulative effects of support strategies across multi-turn interactions.

However, these methods still face several challenges: \textbf{(i) Context-sensitive strategy selection.} The suitability of a support strategy depends on the current user input, dialogue stage, interaction history, and evolving emotional state. Insufficient modeling of this dynamic context limits adaptation in strategy type, intervention timing, and support intensity. \textbf{(ii) Long-horizon credit assignment.} The effect of a strategy may emerge only after several turns, making terminal outcomes and single-turn preferences insufficient for assigning credit to earlier states and decisions or guiding strategy selection across dialogue stages. \textbf{(iii) Affective evolution modeling.} Current objectives mainly evaluate individual responses or complete dialogue outcomes without explicitly modeling affective changes across strategy transitions. This weakens the temporal connections among support strategies, dialogue states, and subsequent affective outcomes, reducing long-horizon coherence and alignment with psychological support principles.

To address these challenges, we propose AFlow, a framework for multi-turn emotional support conversation, as shown in Figure 1(c). AFlow constructs a search tree by exploring alternative support-strategy paths and estimates the utility of intermediate dialogue states and strategy decisions, yielding process-level information beyond sparse terminal rewards. Using the state and state–strategy assessments obtained from the search tree, Affective Flow Preference Optimization (AFPO) learns how affective flow evolves across dialogue states under different support strategies. By enforcing subpath-level flow-balance constraints, AFPO propagates downstream support signals to earlier dialogue states and strategy decisions, thereby establishing a connection between long-term affective outcomes and cross-turn strategy selection (a formal proof is provided in
Appendix~\ref{app:credit-propagation}).

The main contributions of this work are as follows: (1) We formulate multi-turn emotional support as a nonnegative affective utility flow over dialogue states and strategy transitions, establishing a principled connection between long-term support outcomes and intermediate dialogue states and strategy decisions. (2) We propose AFPO, which learns affective-flow evolution from searched trajectories and propagates downstream preference signals through subpath-level flow balance, providing continuous and structurally consistent process supervision for multi-turn strategy learning. (3) We conduct systematic experiments on ExTES and ESConv, demonstrating that AFlow consistently improves macro-F1 for strategy prediction and response diversity while maintaining competitive response generation quality.

\section{Related Work}
\label{sec:related}

Emotional support conversation research has evolved from supervised strategy modeling and structured response generation~\cite{deng2023knowledge,zhang2024escot} to LLM-based prompting, data augmentation, and preference alignment~\cite{zheng2024self,kang2024can,sun2025rethinking}. However, these methods largely rely on response- or dialogue-level supervision, limiting their ability to capture the long-term effects of support strategies across turns. Recent studies on strategy planning and affect dynamics~\cite{wan2025emodynamix,zhao2025chain,madani2025steering,shen2025flow} model finer-grained interaction processes, but seldom connect intermediate strategy decisions with subsequent affective evolution. AFlow addresses this gap by learning affective-flow evolution through subpath flow optimization over searched strategy trajectories. Further discussion is provided in Appendix~\ref{app:related}.

\section{Preliminaries}
\label{sec:prelim}

Multi-turn Emotional Support Conversation (ESC) is modeled as a Markov Decision Process (MDP). Let $s_t \in \mathcal{S}$ denote the dialogue state at turn $t$, which encodes the dialogue history up to that turn, and let $a_t \in \mathcal{A}$ denote the supporter action at turn $t$. The environment transitions according to $p(s_{t+1}\mid s_t,a_t)$. A conversation with horizon $T$ is represented as a finite interaction trajectory $\tau = (s_0, a_0, s_1, a_1, \ldots, s_T), \quad s_T \in \mathcal{S}_{\mathrm{term}}$, where $\mathcal{S}_{\mathrm{term}}$ denotes the set of terminal states. Each terminal state $s_T$ is assigned a non-negative terminal affective reward $R(s_T)\ge 0$, which evaluates the completed interaction from the perspective of emotional support.

Generative Flow Networks (GFlowNets)~\cite{bengio2023gflownet} model sampling over discrete objects through sequential construction guided by a flow-balance objective. The generation process is represented as a directed graph whose nodes are partial objects (states) and edges are one-step extensions. Let $\mathrm{Pred}(s)$ and $\mathrm{Succ}(s)$ denote the predecessor and successor sets of $s$. Each transition is associated with a non-negative edge flow $F(s \rightarrow s')\ge 0$. For any non-terminal state $s$, flow conservation requires that the total incoming flow equals the total outgoing flow:
\begin{equation}
\sum_{s' \in \mathrm{Pred}(s)} F(s' \rightarrow s)
\;=\;
\sum_{s'' \in \mathrm{Succ}(s)} F(s \rightarrow s'').
\end{equation}
For a terminal state $s_T \in \mathcal{S}_{\mathrm{term}}$, the incoming flow is absorbed and matched to a non-negative terminal reward:
$\sum_{s' \in \mathrm{Pred}(s_T)} F(s' \rightarrow s_T) = R(s_T)$ with $R(s_T) \ge 0$.

\section{Methodology}
\label{sec:method}

In this section, we present the Affective Flow Language Model (AFlow) (Figure~\ref{fig:aflow}). AFlow consists of three stages: (i) affective trajectory construction, which explores alternative strategy-conditioned dialogue continuations and derives psychology-grounded supervision for intermediate decisions; (ii) Affective Flow Preference Optimization (AFPO), which converts the resulting state and strategy values into flow-based training signals and propagates preferences across dialogue subpaths; and (iii) Affective Reasoning, which performs bounded lookahead with the learned policy and evaluation model at inference time.

\begin{figure*}[t]
    \centering
    \includegraphics[width=\textwidth]{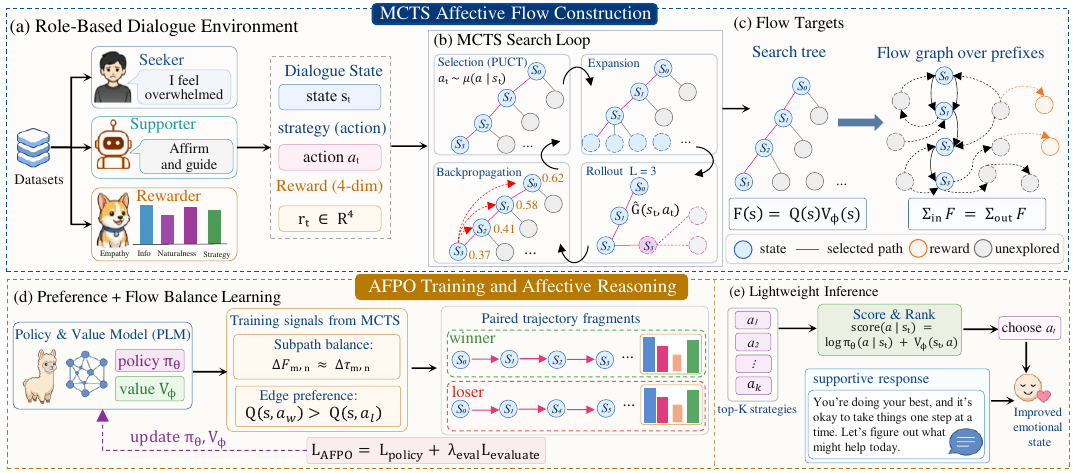}
    \caption{Detailed diagram of the AFlow framework for emotional support conversation.}
    \label{fig:aflow}
\end{figure*}

\subsection{Affective Flow Reward Acquisition}
\label{subsec:mcts}

In AFlow, we use a role-separated LLM environment together 
with MCTS to construct fine-grained affective reward signals for multi-turn ESC, and transform these signals into supervision for Affective Flow optimization.

\paragraph{Role-based dialogue environment.}
Tree expansion uses three role-specific agents: a seeker LLM, a supporter LLM, and a rewarder LLM, each with dedicated prompts and independent dialogue contexts. At dialogue state $s_t$ (Figure~\ref{fig:aflow}), representing the history up to turn $t$, the strategy policy produces a prior distribution $\pi_\theta(a \mid s_t)$ over $a \in \mathcal{A}$. Given strategy $a_t$, the supporter generates a response, and the seeker simulates the subsequent user utterance, transitioning to $s_{t+1}$. Each expanded state--strategy pair adds one sampled supporter--seeker continuation as its successor, while alternatives form separate searched branches. The rewarder evaluates the supporter response on Empathy, Information Quality, Human-like Naturalness, and Strategic Efficacy to produce an immediate reward $r_t$. The full environment specification is provided in Appendix~\ref{app:environment}.

\paragraph{Affective trajectory search.} AFlow constructs an affective search tree through four operations: selection, expansion, rollout, and backpropagation. The resulting trajectories provide utility estimates for intermediate dialogue states and strategy decisions, serving as supervision for AFPO.

\paragraph{Selection.}
The exploration ratio $\rho_t\in[0,1]$ controls the uniform exploration weight and is defined as

\begin{equation}
\rho_t =
\frac{\kappa}{\kappa+\log\!\big(N(s_t)+1\big)},
\end{equation}
where $\kappa=0.1$ is a hyperparameter and $N(s)=\sum_{a\in\mathcal{A}}N(s,a)$ denotes the state visit count. The strategy sampling distribution $\mu(\cdot\mid s_t)$ is then constructed as

\begin{equation}
\begin{aligned}
&\mu(a \mid s_t)
=
(1-\rho_t)\,\mathrm{Softmax}(u_t(a))
+\rho_t\cdot \mathrm{U}(\mathcal{A})(a),\\
& u_t(a)
=
Q(s_t,a)
+c_{\mathrm{puct}}\,\pi_\theta(a \mid s_t)\,
\sqrt{\frac{N(s_t)}{1+N(s_t,a)}},
\end{aligned}
\label{eq:action-dist}
\end{equation}
where $Q(s_t,\cdot)$ denotes action-value estimates, $\pi_\theta(\cdot\mid s_t)$ is the policy prior, $c_{\mathrm{puct}}$ controls the exploration--exploitation trade-off, $N(s_t,\cdot)$ denotes state--action visit counts, and $\mathrm{U}(\mathcal{A})(a)=1/|\mathcal{A}|$ denotes the uniform distribution over $\mathcal{A}$. The strategy $a_t$ is sampled from $\mu(\cdot\mid s_t)$.

\paragraph{Expansion.}
Given the selected pair $(s_t,a_t)$, the successor dialogue state $s_{t+1}$ is instantiated using the role-based simulation described above: the supporter generates the next supporter utterance conditioned on $(s_t,a_t)$, and the seeker generates the subsequent user utterance to advance the dialogue. The resulting dialogue history is stored at $s_{t+1}$ for subsequent selection and rollout. When a state is expanded for the first time, statistics for all strategy actions in $\mathcal{A}$ are initialized with $Q=0$ and $N=0$ to ensure consistent updates across branches.

\paragraph{Simulation/rollout.}
A bounded-depth rollout with maximum depth $L=3$ is used to estimate the downstream utility of taking $a_t$ at $s_t$. Low-quality premature termination is suppressed by assigning zero return when the designated premature-end token is generated during rollout. The immediate reward is defined as
\begin{equation}
r_t =
\frac{w_E E_t + w_I I_t + w_H H_t + w_S S_t}{\ell},
\label{eq:immediate-reward}
\end{equation}
where $E_t,I_t,H_t,S_t\in[0,5]$ denote empathy, information quality,
human-like naturalness, and strategic efficacy. We use
$(w_E,w_I,w_H,w_S)=(0.1,0.1,0.1,0.2)$, assigning a larger weight to
the composite strategy score, and set $\ell=2.5$ to normalize
$r_t$ to $[0,1]$.

The rollout then proceeds under $\pi_\theta(a\mid s)$ and collects $r_{t+1},\ldots,r_{t+L'}$ with realized depth $L'$. These rewards measure the downstream affective quality obtained after expanding the selected state--strategy pair $(s_t,a_t)$. Let $I_{\mathrm{terminal}}=0$ if the designated premature-end token $a_{\mathrm{end}}$ is generated during rollout, and $I_{\mathrm{terminal}}=1$ otherwise. Then
\begin{equation}
\hat{G}(s_t, a_t)
=
I_{\mathrm{terminal}}\cdot
\frac{1}{L'}
\sum_{k=1}^{L'} r_{t+k}.
\end{equation}
AFPO flow-balance optimization retains valid expanded nodes with positive MCTS values and excludes zero-return branches triggered by the designated premature-end token.

\paragraph{Backpropagation.}
The return estimate $\hat{G}(s_t,a_t)$ is propagated along the selected path to update $N(s,a)$ and $Q(s,a)$ for all visited pairs $(s_i,a_i)$. The visit count is incremented, and $Q$ is updated by an incremental mean:
\begin{equation}
N(s_i,a_i)\leftarrow N(s_i,a_i)+1,
\end{equation}
\begin{equation}
Q(s_i,a_i)\leftarrow Q(s_i,a_i)+
\frac{1}{N(s_i,a_i)}\big(\hat{G}(s_t,a_t)-Q(s_i,a_i)\big).
\end{equation}
During backpropagation, the node value $Q(s)$ is synchronously updated from the same rollout returns used to update its state--strategy statistics. $Q(s)$ and $Q(s,a)$ therefore provide state-level and strategy-level supervision for AFPO, respectively. The resulting $\{N(s,a),Q(s),Q(s,a)\}$ are retained for the subsequent optimization stage.

\subsection{Affective Flow Preference Optimization}
\label{sec:afpo}

AFPO uses the search-tree estimates $Q(s)$ and $Q(s,a)$ to model affective-flow evolution over strategy-conditioned dialogue transitions.

\paragraph{Optimization variables.}
AFPO optimizes $\pi_\theta(a\mid s)$ and $V_\phi$. The evaluation model produces state valuations $V_\phi(s)$ for flow construction and strategy-conditioned valuations $V_\phi(s,a)$ for preference learning.

\paragraph{Affective Flow parameterization.}
The state flow associated with a dialogue state $s$ is defined as
\begin{equation}
F(s)=Q(s)\,V_\phi(s),
\label{eq:afpo_flow}
\end{equation}
where $Q(s)$ is the MCTS-derived state value estimate at $s$, and $V_\phi(s)$ provides a learnable non-negative valuation. $Q(s)$ embeds downstream affective rewards backpropagated from searched continuations, while $V_\phi(s)$ captures the overall affective quality of the dialogue state.

\paragraph{Flow construction in search.}
The searched trajectories form an affective tree whose nodes are path-specific dialogue prefixes and whose edges are one-step strategy-conditioned transitions. Let $F(s\!\rightarrow\! s')\ge 0$ denote the edge flow, and define the state flow as the total incoming flow:
\begin{equation}
F(s)=\sum_{s'\in \mathrm{Pred}(s)} F(s'\!\rightarrow\! s),
\label{eq:afpo_state_inflow}
\end{equation}
where $\mathrm{Pred}(s)$ is the set of predecessor states of $s$ in the search tree. For non-terminal nodes, flow conservation holds:
\begin{equation}
\sum_{s'\in \mathrm{Pred}(s)} F(s'\!\rightarrow\! s)
=
\sum_{s''\in \mathrm{Succ}(s)} F(s\!\rightarrow\! s''),
\label{eq:afpo_conservation}
\end{equation}
with $\mathrm{Succ}(s)$ denoting the successor set.

To connect flow to the policy, the edge flow on a visited transition is factorized by the forward strategy likelihood:
\begin{equation}
F(s_i\!\rightarrow\! s_{i+1}) = F(s_i)\,\pi_\theta(a_i\mid s_i),
\label{eq:afpo_edge_forward}
\end{equation}
where $a_i$ is the strategy selected on the recorded edge from $s_i$ to $s_{i+1}$. The recorded successor state $s_{i+1}$ already instantiates the generated continuation between the supporter and seeker, while different realized continuations form separate branches in the affective tree.

For training-time balancing, a backward policy $\pi_B(\cdot\mid s_{i+1})$ over predecessors is introduced so that the same edge flow can also be written as
\begin{equation}
F(s_i\!\rightarrow\! s_{i+1}) = F(s_{i+1})\,\pi_B(s_i\mid s_{i+1}).
\label{eq:afpo_edge_backward}
\end{equation}
Equating Eq.~(\ref{eq:afpo_edge_forward}) and Eq.~(\ref{eq:afpo_edge_backward}) yields the edge-wise flow matching:
\begin{equation}
F(s_i)\,\pi_\theta(a_i\mid s_i)=F(s_{i+1})\,\pi_B(s_i\mid s_{i+1}).
\label{eq:afpo_edge_match}
\end{equation}

\paragraph{Subpath flow-balance constraint.}
Consider a searched trajectory and any subpath from $s_m$ to $s_n$ with $m<n$, together with the intervening strategy sequence $\{a_i\}_{i=m}^{n-1}$. Multiplying Eq.~(\ref{eq:afpo_edge_match}) along the subpath yields the general subpath balance:
\begin{equation}
F(s_m)\prod_{i=m}^{n-1}\pi_\theta(a_i\mid s_i)
=
F(s_n)\prod_{i=m}^{n-1}\pi_B(s_i\mid s_{i+1}).
\label{eq:afpo_subpath_general}
\end{equation}
The derivation is provided in Appendix~\ref{app:subpath-flow-derivation}. Since each non-root node has a unique predecessor, $\pi_B(s_i\mid s_{i+1})=1$, and Eq.~(\ref{eq:afpo_subpath_general}) reduces to
\begin{equation}
F(s_n)=F(s_m)\prod_{i=m}^{n-1}\pi_\theta(a_i\mid s_i).
\label{eq:afpo_subpath}
\end{equation}
Replacing $F(s_m)$ and $F(s_n)$ in Eq.~(\ref{eq:afpo_subpath}) with the state-flow definition $F(s)=Q(s)V_\phi(s)$ gives
\begin{equation}
Q(s_n)V_\phi(s_n)
=
Q(s_m)V_\phi(s_m)\prod_{i=m}^{n-1}\pi_\theta(a_i\mid s_i).
\label{eq:afpo_qv_subpath}
\end{equation}
Taking logarithms gives
\begin{equation}
\log\!\big(Q(s_n)V_\phi(s_n)\big)
-
\log\!\big(Q(s_m)V_\phi(s_m)\big)
=
\sum_{i=m}^{n-1}\log \pi_\theta(a_i\mid s_i).
\label{eq:afpo_log_subpath}
\end{equation}
To streamline the formulation of the learning objective, we introduce the shorthand
\begin{equation}
\Delta F_{m,n}
=
\log\!\big(Q(s_n)V_\phi(s_n)\big)
-
\log\!\big(Q(s_m)V_\phi(s_m)\big),
\label{eq:afpo_deltaF}
\end{equation}
\begin{equation}
\Delta \pi_{m,n}
=
\sum_{i=m}^{n-1}\log \pi_\theta(a_i\mid s_i).
\label{eq:afpo_deltaPi}
\end{equation}

\paragraph{Policy objective.}
Let $\mathcal{T}$ denote the set of trajectories explored by the search procedure. Subpath flow balance is enforced by penalizing deviations between $\Delta F_{m,n}$ and $\Delta \pi_{m,n}$ over subpaths:
\begin{equation}
\mathcal{L}_{\mathrm{policy}}
=
\sum_{\tau\in\mathcal{T}}
\sum_{0\le m<n\le T_\tau}
\big(\Delta F_{m,n}-\Delta \pi_{m,n}\big)^2,
\label{eq:afpo_Lpolicy}
\end{equation}
where $T_\tau$ is the horizon of trajectory $\tau$. The loss is applied to all subpath pairs $(m,n)$ within each trajectory, yielding $O(T_\tau^2)$ flow-balance training signals per trajectory. Since the conversation horizon is bounded, the quadratic subpath construction remains tractable in our setting.

\newcolumntype{N}{D{.}{.}{2.2}}

\begin{table*}[t]
\centering
\footnotesize
\setlength{\tabcolsep}{1.5pt}
\renewcommand{\arraystretch}{1.0}

\resizebox{\textwidth}{!}{%
\begin{tabular}{ccl|NN|NNNNN|NNN|NNNN|N}
\toprule
& & &
\multicolumn{2}{c|}{General LLMs} &
\multicolumn{5}{c|}{Sim., Values \& Empathy} &
\multicolumn{3}{c|}{Task \& Interp.} &
\multicolumn{4}{c|}{Strat. \& Know.} &
\multicolumn{1}{c}{Ours} \\
\cmidrule(lr){4-5}
\cmidrule(lr){6-10}
\cmidrule(lr){11-13}
\cmidrule(lr){14-17}
\cmidrule(lr){18-18}
Data & Cat. & Metric
& \multicolumn{1}{c}{GPT-4o}
& \multicolumn{1}{c|}{\makecell{Claude\\3.5}}
& \multicolumn{1}{c}{MoEL}
& \multicolumn{1}{c}{MIME}
& \multicolumn{1}{c}{\makecell{Social\\Sim}}
& \multicolumn{1}{c}{\makecell{Multi\\Agent}}
& \multicolumn{1}{c|}{\makecell{Value\\Reinf.}}
& \multicolumn{1}{c}{\makecell{Blender\\Bot-Joint}}
& \multicolumn{1}{c}{ESCoT}
& \multicolumn{1}{c|}{\makecell{Intent\\ESC}}
& \multicolumn{1}{c}{MISC}
& \multicolumn{1}{c}{KEMI}
& \multicolumn{1}{c}{EmoDx}
& \multicolumn{1}{c|}{DYN-Kn}
& \multicolumn{1}{c}{AFlow} \\
\midrule

\multirow{9}{*}{\rotatebox{90}{ExTES}}
& Fluency & PPL$\downarrow$
& \multicolumn{1}{c}{--}
& \multicolumn{1}{c|}{--}
& \multicolumn{1}{c}{--}
& \multicolumn{1}{c}{--}
& 10.23 & 9.85 & 9.92 & 16.51 & 11.21 & 10.46
& 15.57 & 14.77 & 10.06 & 10.31  & 9.15 \\
\cmidrule{2-18}

& \multirow{4}{*}{Gener.} & B-2$\uparrow$
& 15.81 & 16.48 & 8.47 & 9.16 & 18.52 & 19.81 & 19.14
& 11.36 & 17.49 & 18.23 & 13.77 & 15.49 & 19.52 & 19.18 & 22.44 \\

& & B-4$\uparrow$
& 6.54 & 7.26 & 1.84 & 2.09 & 6.75 & 7.45 & 6.92
& 2.24 & 5.35 & 6.13 & 2.74 & 3.19 & 7.36 & 7.06 & 9.05 \\

& & R-L$\uparrow$
& 19.18 & 20.81 & 14.46 & 15.15 & 20.86 & 21.79 & 20.52
& 16.46 & 19.52 & 20.19 & 16.07& 17.48 & 21.45 & 21.08 & 24.54 \\

& & MET$\uparrow$
& 23.81 & 23.49 & 11.22 & 11.81 & 20.13 & 21.09 & 20.82
& 13.19 & 18.76 & 19.48 & 14.45 & 15.83 & 20.89 & 20.51 & 23.12 \\
\cmidrule{2-18}

& \multirow{2}{*}{Divers.} & D-1$\uparrow$
& 4.05 & 4.83 & 1.24 & 1.36 & 4.09 & 4.24 & 4.13
& 1.70  & 3.68 & 3.82 & 2.91 & 3.28 & 4.35 & 4.19 &  5.11\\

& & D-2$\uparrow$
& 21.79 & 23.81 & 6.47 & 7.16 & 17.77 & 19.46 & 18.49
& 6.80 & 15.76 & 16.45 & 11.82 & 13.52 & 19.08 & 18.76 & 21.05 \\
\cmidrule{2-18}

& \multirow{2}{*}{Strat.} & Acc$\uparrow$
& 57.78 & 60.51
& \multicolumn{1}{c}{--}
& \multicolumn{1}{c}{--}
& 58.74 & 60.17 & 59.45 & 30.53 & 55.41 & 58.15
& 45.11 & 49.52 & 60.83 & 59.51 & 64.49\\

& & F1$\uparrow$
& 54.69 & 58.91
& \multicolumn{1}{c}{--}
& \multicolumn{1}{c}{--}
& 56.16 & 58.46 & 57.07 & 26.79 & 53.16 & 56.52
& 41.82 & 46.70 & 58.88 & 57.81 & 62.84 \\

\midrule

\multirow{9}{*}{\rotatebox{90}{ESConv}}
& Fluency & PPL$\downarrow$
& \multicolumn{1}{c}{--}
& \multicolumn{1}{c|}{--}
& \multicolumn{1}{c}{--}
& \multicolumn{1}{c}{--}
& 9.84 & 9.51 & 9.21 & 16.39 & 10.46 & 9.93
& 15.62 & 14.76 & 9.55 & 9.82 & 8.95 \\
\cmidrule{2-18}

& \multirow{4}{*}{Gener.} & B-2$\uparrow$
& 16.17 & 17.08 & 8.78 & 9.45 & 19.48 & 20.54 & 19.69
& 12.11 & 18.23 & 19.14 & 14.51 & 16.06 & 20.09& 19.70 & 22.85 \\

& & B-4$\uparrow$
& 6.84 & 7.53 & 1.95 & 2.21 & 7.20 & 7.84 & 7.32
& 2.53 & 5.62 & 6.95 & 2.96 & 3.43 & 7.62 & 7.40 & 9.25 \\

& & R-L$\uparrow$
& 19.65 & 21.30 & 14.72 & 15.43 & 21.59 & 22.49 & 21.23
& 17.32 & 20.13 & 21.21 & 16.83 & 18.17 & 22.09 & 21.78 & 24.86 \\

& & MET$\uparrow$
& 24.07 & 24.18 & 11.46 & 12.12 & 20.89 & 21.47 & 21.00
& 13.82 & 19.15 & 20.61 & 15.09 & 16.16 & 21.25 & 21.02 & 23.55 \\
\cmidrule{2-18}

& \multirow{2}{*}{Divers.} & D-1$\uparrow$
& 4.12 & 4.95 & 1.31 & 1.43 & 4.25 & 4.46 & 4.33
& 1.83 & 3.86 & 4.17 & 3.12 & 3.42 & 4.51 & 4.28 & 5.27 \\

& & D-2$\uparrow$
& 22.08& 24.82 & 6.83 & 7.47 & 18.52 & 19.49 & 18.91
& 7.18 & 16.52 & 17.82 & 12.47 & 14.12 & 19.78 & 18.89 & 21.53 \\
\cmidrule{2-18}

& \multirow{2}{*}{Strat.} & Acc$\uparrow$
& 58.41 & 61.22
& \multicolumn{1}{c}{--}
& \multicolumn{1}{c}{--}
& 59.82 & 62.00 & 60.19 & 32.07 & 56.78 & 59.46
& 46.51 & 51.16 & 61.76 & 60.47 & 65.10 \\

& & F1$\uparrow$
& 55.20 & 59.50
& \multicolumn{1}{c}{--}
& \multicolumn{1}{c}{--}
& 57.51 & 60.46 & 58.48 & 28.50 & 54.52 & 57.83
& 43.16 & 49.78 & 60.12 & 58.84 & 63.58 \\
\bottomrule
\end{tabular}
}
\caption{Automatic evaluation results. PPL is computed using the same external language-model scorer for all available system outputs, ensuring direct comparability across methods. PPL is omitted only for methods whose generated outputs are unavailable.}
\label{tab:main_results}
\end{table*}

\paragraph{Evaluation objective.}
Edge-level preferences are induced by the empirical edge estimates at the same dialogue state. For a visited state $s_t$, a pair of strategies $(a_w,a_l)$ is formed such that
\begin{equation}
Q(s_t,a_w)>Q(s_t,a_l).
\label{eq:afpo_pref_pair}
\end{equation}
Under the bounded-rollout assumptions stated in
Appendix~\ref{app:reward-grounding}, the probability that this
search-derived ordering reverses the expected downstream-utility
ordering decreases as additional rollout evidence is collected.

A margin-based ranking loss encourages consistent ordering:
\begin{equation}
\begin{aligned}
\mathcal{L}_{\mathrm{evaluate}}
&=
\mathbb{E}_{(s_t,a_w,a_l)} \\
&\quad
\max\!\Big(0,\gamma-\big(V_\phi(s_t,a_w)-V_\phi(s_t,a_l)\big)\Big),
\end{aligned}
\label{eq:afpo_Leval}
\end{equation}
where $\gamma>0$ is a margin hyperparameter.

\paragraph{Overall objective.}
The AFPO training objective is
\begin{equation}
\mathcal{L}_{\mathrm{AFPO}}
=
\mathcal{L}_{\mathrm{policy}}
+
\lambda_{\mathrm{eval}}\,\mathcal{L}_{\mathrm{evaluate}},
\label{eq:afpo_Ltotal}
\end{equation}
where $\lambda_{\mathrm{eval}}$ controls the trade-off between subpath flow-balance supervision and edge-level preference supervision.

\paragraph{Training procedure.}
Training alternates among trajectory construction, flow computation, and model updates. In each round, the policy generates affective trajectories and their state and state--strategy values, while the evaluation model provides learnable valuations for flow targets. Minimizing $\mathcal{L}_{\mathrm{AFPO}}$ updates both models for the next round, progressively aligning strategy decisions with downstream support outcomes. AFPO also preserves diverse support paths through reward-proportional trajectory matching (Appendix~\ref{app:afpo-diversity}).

\subsection{Affective Reasoning}
\label{subsec:inference}

After training, AFlow replaces full tree construction with Affective Beam Search (ABS), a bounded lookahead procedure guided by the learned policy and evaluation model. Starting from the observed dialogue state, ABS maintains a small beam of promising simulated dialogue states. For each state in the beam, the policy proposes a fixed number of candidate support strategies in parallel. The supporter and seeker agents instantiate the corresponding one-step continuations, and the evaluation model scores each extension together with its policy likelihood. Only the highest-ranked states are retained for the next lookahead step. The process continues until reaching the prescribed lookahead depth or a terminal state. The first strategy on the highest-scoring trajectory is executed in the actual conversation, after which the next user utterance defines a new search root. By restricting both beam width and lookahead depth, ABS preserves multi-path foresight while avoiding the cost of reconstructing the full affective tree at every turn. Analysis is provided in Appendice~\ref{app:depth}.


\newcolumntype{C}{>{\hspace{0.45pt}}c<{\hspace{0.45pt}}}

\begin{table*}[t]
\centering
\scriptsize
\setlength{\tabcolsep}{1.55pt}
\renewcommand{\arraystretch}{1.05}
\setlength{\arrayrulewidth}{0.4pt}
\setlength{\aboverulesep}{0pt}
\setlength{\belowrulesep}{0pt}

\resizebox{\textwidth}{!}{%
\begin{tabular}{@{}l|l|CCCC|CC|CC|l|l|CCCCC|CCCCC@{}}
\toprule

\multicolumn{10}{c|}{
(a) Robustness across diverse LLM backbones}
&
\multicolumn{12}{c}{
(b) Pairwise preference evaluation (Win/Tie/Lose \%)}
\\

\midrule

\multirow{2}{*}{Environment}
& \multirow{2}{*}{Backbone}
& \multicolumn{4}{c|}{Generation Quality}
& \multicolumn{2}{c|}{Diversity}
& \multicolumn{2}{c|}{Strategy}
& \multirow{2}{*}{Method}
& \multirow{2}{*}{Result}
& \multicolumn{5}{c|}{GPT-5.2 Evaluation}
& \multicolumn{5}{c}{Human Evaluation}
\\

\cmidrule(lr){3-6}
\cmidrule(lr){7-8}
\cmidrule(lr){9-10}
\cmidrule(lr){13-17}
\cmidrule(lr){18-22}

& & B-2 & B-4 & R-L & MET
& D-1 & D-2 & Acc & F1
& &
& Flu. & Iden. & Comf. & Sugg. & Ovr.
& Flu. & Iden. & Comf. & Sugg. & Ovr.
\\

\midrule

\multirow{3}{*}{GPT-4o}
& Qwen-2.5
& 22.44 & 9.05 & 24.54 & 23.12
& 5.11 & 21.05 & 64.49 & 62.84
& \multirow{3}{*}{ESCoT}
& Win
& 35.5 & 58.5 & 65.2 & 55.4 & 62.5
& 32.0 & 55.0 & 60.5 & 52.5 & 58.8
\\

& Gemma-2
& 22.05 & 8.79 & 23.94 & 22.66
& 4.95 & 20.40 & 63.81 & 62.09
&
& Tie
& 55.2 & 28.2 & 20.5 & 26.8 & 22.5
& 58.5 & 30.5 & 24.0 & 29.0 & 25.2
\\

& LLaMA-3.1
& 21.61 & 8.45 & 23.39 & 22.15
& 4.82 & 19.77 & 63.02 & 61.53
&
& Lose
& 9.3 & 13.3 & 14.3 & 17.8 & 15.0
& 9.5 & 14.5 & 15.5 & 18.5 & 16.0
\\

\midrule

\multirow{3}{*}{Claude-3.5}
& Qwen-2.5
& 22.72 & 9.18 & 24.86 & 23.44
& 5.22 & 21.39 & 65.19 & 63.49
& \multirow{3}{*}{EmoDx}
& Win
& 30.5 & 50.5 & 45.5 & 68.5 & 58.2
& 28.5 & 48.0 & 42.5 & 64.2 & 54.5
\\

& Gemma-2
& 22.27 & 8.95 & 24.31 & 22.95
& 5.05 & 20.73 & 64.36 & 62.75
&
& Tie
& 62.0 & 35.5 & 38.2 & 18.5 & 25.5
& 64.0 & 37.5 & 40.5 & 21.5 & 28.5
\\

& LLaMA-3.1
& 21.85 & 8.65 & 23.83 & 22.38
& 4.92 & 20.15 & 63.58& 62.20
&
& Lose
& 7.5 & 14.0 & 16.3 & 13.0 & 16.3
& 7.5 & 14.5 & 17.0 & 14.3 & 17.0
\\

\midrule

\multirow{3}{*}{Gemini-1.5}
& Qwen-2.5
& 22.34 & 8.95 & 24.43 & 23.00
& 5.05 & 20.88 & 64.28 & 62.47
& \multirow{3}{*}{GPT-4o}
& Win
& 15.5 & 48.5 & 48.2 & 52.5 & 50.5
& 12.5 & 45.5 & 45.8 & 49.2 & 48.5
\\

& Gemma-2
& 21.95 & 8.73 & 23.85 & 22.54
& 4.91 & 20.25 & 63.46 & 61.75
&
& Tie
& 70.2 & 38.5 & 36.5 & 32.5 & 33.5
& 74.0 & 41.5 & 38.5 & 35.5 & 35.2
\\

& LLaMA-3.1
& 21.50 & 8.34 & 23.26 & 22.05
& 4.74 & 19.65 & 62.77 & 61.20
&
& Lose
& 14.3 & 13.0 & 15.3 & 15.0 & 16.0
& 13.5 & 13.0 & 15.7 & 15.3 & 16.3
\\

\bottomrule
\end{tabular}%
}
\caption{Evaluation of AFlow in terms of backbone robustness and
pairwise preference against representative baselines.}
\label{tab:combined_evaluation}
\end{table*}

\section{Experiments}
\label{sec:experiments}


In this section, we present the experimental setup and results through six research questions (RQs). RQ1: How does AFlow compare with existing methods? RQ2: Can AFPO distinguish high- and low-utility branches in multi-turn ESC? RQ3: Is AFlow effective across the Exploration, Comforting, and Action stages? RQ4: Does AFlow exhibit stable, improving turn-wise affective reward trajectories? RQ5: Does AFlow maintain diverse support strategies across turns? RQ6: Does each core component contribute to AFlow's performance?

\begin{figure}[t]
    \centering
    \includegraphics[width=1.0\linewidth]{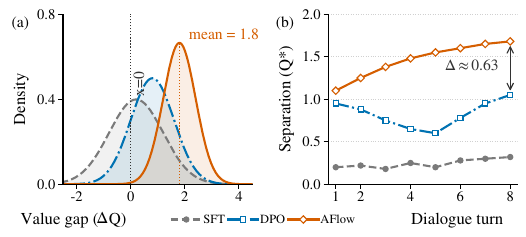}
    \caption{\textbf{(a)} Value-gap distributions across methods. \textbf{(b)} Turn-wise value separation across dialogue turns.}
    \label{fig:rq2-fig}
\end{figure}

\subsection{Experimental Setup}
\paragraph{Datasets.} We use two in-domain ESC datasets as seed data: ExTES~\cite{zheng2023building} and ESConv~\cite{liu2021towards}. To reduce the search space, semantically similar support strategies are consolidated into eight categories. We sample 1{,}000 dialogues from ExTES and 500 from ESConv while covering all emotional support scenarios.

\paragraph{Baselines.} We compare AFlow with representative methods from three categories. 1) Task-oriented and interpretable ESC systems, including BlenderBot-Joint~\cite{liu2021towards}, ESCoT~\cite{zhang2024escot}, and IntentionESC~\cite{zhang2025intentionesc}. 2) Strategy, knowledge, and persona enhanced models, including MISC~\cite{tu2022misc}, KEMI~\cite{deng2023knowledge}, EmoDynamiX~\cite{wan2025emodynamix}, and dynamic knowledge filtering with persona extraction~\cite{hao2025enhancing}. 3) Test-time, simulation-driven, and empathetic generation methods, including MultiAgentESC~\cite{xu2025multiagentesc}, SocialSim~\cite{chen2025socialsim}, Value Reinforcement~\cite{kim2025dialogue}, MoEL~\cite{lin2019moel}, and MIME~\cite{majumder2020mime}. We also include GPT-4o and Claude-3.5 as strong proprietary baselines.

\paragraph{Models and Training.} AFlow is primarily implemented with Qwen-2.5-7B and further evaluated with LLaMA-3.1-8B and Gemma-2. Training proceeds iteratively: the backbone performs MCTS in the role-separated dialogue environment to construct affective trajectory trees, and the resulting data optimize $(\pi_{\theta},V_{\phi})$ under the AFPO objective.

\paragraph{Evaluation.} We report automatic metrics for generation quality, diversity, fluency, and strategy alignment. We additionally conduct pairwise preference evaluation with GPT-5.2 and human experts. Further experimental and reproducibility details are provided in Appendix~\ref{app:experimental_details}.

\subsection{Main Results (RQ1)}

As shown in Table~\ref{tab:main_results}, AFlow achieves the best strategy alignment on the expanded ESConv and ExTES validation dialogues, while remaining competitive in response diversity, generation quality, and perplexity. Table~\ref{tab:combined_evaluation}~(a) further demonstrates its stable performance across different backbone models and simulation environments. The GPT-5.2 and human pairwise evaluations in Table~\ref{tab:combined_evaluation}~(b) show that AFlow outperforms representative task-specific baselines in fluency, identification, comfort, suggestion, and overall preference, while remaining competitive with GPT-4o. Overall, AFlow demonstrates consistent effectiveness across automatic evaluation, preference evaluation, and cross-model settings.  The detailed evaluation rubric and prompts are provided in Appendix~\ref{app:human_eval}.

\begin{figure}[t]
    \centering
    \includegraphics[width=1.0\linewidth]{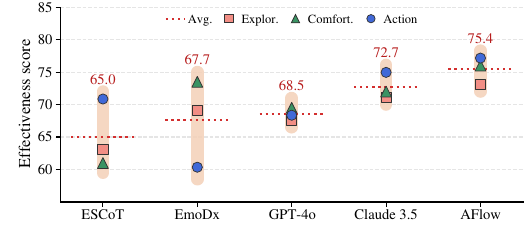}
    \caption{Stage-wise effectiveness scores. AFlow achieves the strongest overall performance.}
    \label{fig:rq3-fig}
\end{figure}

\subsection{Affective State Discrimination (RQ2)}

We compare AFlow with SFT+MCTS and DPO+MCTS to evaluate whether AFPO improves the discrimination between higher- and lower-utility branches. Figure~\ref{fig:rq2-fig}(a) reports the value gap $\Delta Q=Q(s_t,a_{\mathrm{good}})-Q(s_t,a_{\mathrm{bad}})$, where the two strategies correspond to higher and lower terminal rewards, respectively. AFlow shows the strongest positive shift, indicating a clearer preference for higher-utility strategies. Figure~\ref{fig:rq2-fig}(b) further reports the turn-wise value gap $\Delta Q_t^{*}=Q_t^{+}-Q_t^{-}$. AFlow maintains larger and more stable gaps across turns, showing that subpath flow-balance supervision improves intermediate strategy discrimination.

\subsection{Effectiveness across Dialogue Stages (RQ3)}

As shown in Figure~\ref{fig:rq3-fig}, the baselines exhibit clear performance differences across the Exploration, Comforting, and Action stages, and their advantages in earlier stages do not consistently extend to later stages. In contrast, AFlow maintains strong and stable effectiveness across all stages, with more reliable support in the Action stage. Overall, AFlow demonstrates better stage adaptability and sustained problem-oriented support as the dialogue progresses.

\begin{figure}[t]
    \centering
    \includegraphics[width=1.0\linewidth]{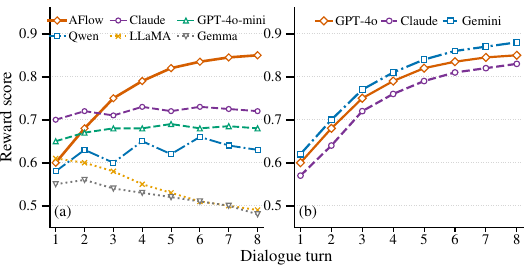}
    \caption{Reward trajectory analysis. AFlow demonstrates stable long-horizon improvement.}
    \label{fig:rq4-fig}
\end{figure}

\begin{figure}[t]
    \centering
    \includegraphics[width=1.0\linewidth]{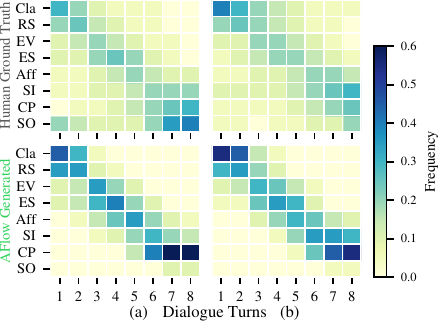}
    \caption{Strategy distribution comparison between human ground truth and AFlow. AFlow preserves expert-level diversity and the Exploration $\rightarrow$ Action progression.}
    \label{fig:rq5-fig}
\end{figure}

\subsection{Affective Flow Trajectory Analysis (RQ4)}

As shown in Figure~\ref{fig:rq4-fig}~(a), we compare the turn-wise reward trajectories of AFlow with several alternative generators. AFlow exhibits a more stable reward trajectory and maintains strong performance in later turns. Figure~\ref{fig:rq4-fig}~(b) further re-scores the same dialogues using three different rewarder models. The consistent turn-wise trends across rewarders indicate that AFlow's long-horizon improvement does not depend on a single reward model.

\subsection{Strategy Diversity Analysis (RQ5)}

As shown in Figure~\ref{fig:rq5-fig}, we compare the turn-wise strategy distributions of AFlow and human reference dialogues across both datasets. AFlow exhibits a strategy transition pattern similar to the human references and maintains coverage of all eight support strategies across turns. These results indicate that AFlow avoids over-concentrating on a small set of strategies and preserves diverse strategy choices as the dialogue progresses, thereby mitigating strategy mode collapse.

\subsection{Ablation Study (RQ6)}

Table~\ref{tab:ablation} confirms the contributions of AFlow's core components, as removing flow balance, MCTS exploration, or process supervision degrades performance. Figure~\ref{fig:rq7-fig} further shows that MCTS and ABS costs scale linearly with their search budgets. Additional analyses and a case study are provided in Appendices~\ref{app:depth} and~\ref{app:case}, respectively.


\begin{table}[t]
\centering
\small
\setlength{\tabcolsep}{3pt}
\renewcommand{\arraystretch}{1.05}
\resizebox{\columnwidth}{!}{
\begin{tabular}{llll}
\toprule
Model Variant & Strategy (F1) & Diversity (D-2) & Quality (B-4) \\
\midrule
AFlow & 63.58 & 21.53 & 9.25 \\
\midrule
w/o Flow Balance & \(55.07_{-8.51}\) & \(16.85_{-4.68}\) & \(7.65_{-1.60}\) \\
w/o MCTS & \(53.46_{-10.12}\) & \(10.22_{-11.31}\) & \(6.95_{-2.30}\) \\
w/o Process Sup. & \(47.49_{-16.09}\) & \(18.45_{-3.08}\)& \(7.37_{-1.88}\) \\
\bottomrule
\end{tabular}
}
\caption{Ablation study. Subscripts denote the performance drop ($\Delta$) relative to full AFlow.}
\label{tab:ablation}
\end{table}

\begin{figure}[t]
    \centering
    \includegraphics[width=1.0\linewidth]{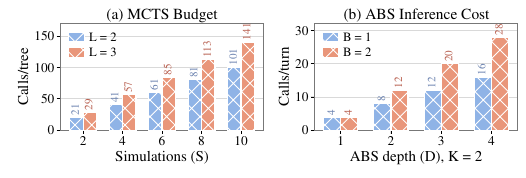}
    \caption{AFlow search cost scales linearly with budget.
}
    \label{fig:rq7-fig}
\end{figure}

\section{Conclusion}
\label{sec:conclusion}

This paper introduces AFlow, a novel framework for aligning large language models in multi-turn emotional support conversations. AFlow constructs affective trajectories through MCTS in a role-separated dialogue environment and employs Affective Flow Preference Optimization (AFPO) with subpath-level flow-balance supervision to provide fine-grained feedback for intermediate dialogue states and strategy decisions. Experiments demonstrate that AFlow improves strategy alignment while preserving diverse strategy usage and competitive generation quality, highlighting the potential of flow-based process supervision for effective intermediate credit assignment in emotional support dialogue.

\bibliography{aaai2027}

\appendix
\include{appendix}

\end{document}

%% file: appendix.tex
\providecommand{\ph}[1]{\texttt{\{#1\}}}

\newcommand{\choicebtn}[1]{\fbox{\scriptsize\strut\,#1\,}}
\newcommand{\UIsep}{\hspace{6pt}}

\providecommand{\H}[1]{\textbf{#1}}
\providecommand{\ph}[1]{\texttt{#1}}
\providecommand{\tokend}{\texttt{<END>}}
\providecommand{\tokenok}{\texttt{<OK>}}
\providecommand{\CardText}{}

\makeatletter
\@ifundefined{PromptCard}{
  \newenvironment{PromptCard}[3]{
    \par\medskip\noindent\textbf{#3}\par\small\medskip
  }{%
    \par\medskip\normalsize
  }
}{}
\makeatother


\section{Proof: Long-Horizon Affective Credit Propagation}
\label{app:credit-propagation}

This appendix provides a formal proof for the long-horizon
affective credit-propagation property introduced in the main text.
We show that, under the AFPO flow-balance formulation, downstream
affective outcomes determine the flow of intermediate dialogue
prefixes and the probabilities assigned to earlier support-strategy
decisions.

\paragraph{Notation.}
Let
$\tau=(s_0,a_0,s_1,a_1,\ldots,s_T)$
denote a complete multi-turn emotional-support trajectory, where
$s_t\in\mathcal{S}$ is the dialogue-prefix state at turn $t$ and
$a_t\in\mathcal{A}$ is the selected support strategy.
The searched trajectories form an affective tree
$\mathcal{G}=(\mathcal{S},\mathcal{E})$, in which every non-root
state has a unique parent.

Let $\mathcal{S}_{\mathrm{term}}$ denote the terminal states, with
non-negative terminal affective reward
$R(z)\geq 0$ for $z\in\mathcal{S}_{\mathrm{term}}$.
For any state $s$, define its reachable terminal descendants as
\begin{equation}
\mathcal{D}_{\mathrm{term}}(s)
=
\{z\in\mathcal{S}_{\mathrm{term}}
\mid z\text{ is reachable from }s\}.
\label{eq:app_terminal_descendants}
\end{equation}

For a strategy-conditioned transition
$s\xrightarrow{a}s'$, AFPO factorizes its edge flow as
\begin{equation}
F(s\rightarrow s')
=
F(s)\pi_\theta(a\mid s).
\label{eq:app_edge_factorization}
\end{equation}
For every non-terminal state, flow conservation requires
\begin{equation}
F(s)
=
\sum_{s'\in\mathrm{Succ}(s)}
F(s\rightarrow s'),
\label{eq:app_flow_conservation}
\end{equation}
while terminal flow is matched to terminal affective reward:
\begin{equation}
F(z)=R(z),
\qquad z\in\mathcal{S}_{\mathrm{term}}.
\label{eq:app_terminal_matching}
\end{equation}
Because every non-root state has a unique parent, the incoming edge
flow of a recorded successor equals its state flow:
\begin{equation}
F(s_i\rightarrow s_{i+1})=F(s_{i+1}).
\label{eq:app_unique_parent}
\end{equation}
The following analysis is restricted to positive-flow states, for
which flow ratios and logarithms are well-defined.

\subsection{Subpath Flow Form}
\label{app:subpath-flow-form}

\begin{lemma}[Subpath Flow Form]
\label{lem:subpath-flow}
For any recorded subpath
$s_m\xrightarrow{a_m}\cdots
\xrightarrow{a_{n-1}}s_n$ with $m<n$, exact flow matching implies
\begin{equation}
\prod_{i=m}^{n-1}\pi_\theta(a_i\mid s_i)
=
\frac{F(s_n)}{F(s_m)}.
\label{eq:app_subpath_ratio}
\end{equation}
Equivalently,
\begin{equation}
\log F(s_n)-\log F(s_m)
=
\sum_{i=m}^{n-1}\log\pi_\theta(a_i\mid s_i).
\label{eq:app_subpath_log}
\end{equation}
\end{lemma}

\begin{proof}
Combining Eqs.~\eqref{eq:app_edge_factorization} and
\eqref{eq:app_unique_parent} gives
\begin{equation}
F(s_{i+1})
=
F(s_i)\pi_\theta(a_i\mid s_i),
\end{equation}
and therefore
\begin{equation}
\pi_\theta(a_i\mid s_i)
=
\frac{F(s_{i+1})}{F(s_i)}.
\end{equation}
Multiplying over $i=m,\ldots,n-1$ yields
\begin{align}
\prod_{i=m}^{n-1}\pi_\theta(a_i\mid s_i)
&=
\prod_{i=m}^{n-1}
\frac{F(s_{i+1})}{F(s_i)}
\nonumber =
\frac{F(s_n)}{F(s_m)},
\end{align}
because all intermediate flow terms cancel telescopically.
Taking logarithms gives Eq.~\eqref{eq:app_subpath_log}.
\end{proof}

\subsection{Long-Horizon Affective Credit Propagation}
\label{app:long-horizon-credit}

\begin{theorem}[Long-Horizon Affective Credit Propagation]
\label{thm:credit-propagation}
Assume exact flow conservation and terminal-reward matching on the
affective search tree. Then the flow of every dialogue state equals
the total terminal affective reward reachable from that state:
\begin{equation}
F(s)
=
\sum_{z\in\mathcal{D}_{\mathrm{term}}(s)}R(z).
\label{eq:app_reward_mass}
\end{equation}

Moreover, if strategy $a$ at state $s$ leads to successor state
$s_a$, then
\begin{equation}
\pi_\theta(a\mid s)
=
\frac{
\sum_{z\in\mathcal{D}_{\mathrm{term}}(s_a)}R(z)
}{
\sum_{z\in\mathcal{D}_{\mathrm{term}}(s)}R(z)
}.
\label{eq:app_strategy_credit}
\end{equation}
Thus, an earlier support strategy receives probability mass in
proportion to the long-term affective reward reachable through its
downstream branch.
\end{theorem}

\begin{proof}
We first prove Eq.~\eqref{eq:app_reward_mass} by backward induction
over the affective search tree.

For a terminal state $z$, terminal-reward matching gives
$F(z)=R(z)$. Since
$\mathcal{D}_{\mathrm{term}}(z)=\{z\}$,
Eq.~\eqref{eq:app_reward_mass} holds at every leaf.

Now consider a non-terminal state $s$ and assume that
Eq.~\eqref{eq:app_reward_mass} holds for every successor
$s'\in\mathrm{Succ}(s)$. By flow conservation and the
unique-parent property,
\begin{align}
F(s)
&=
\sum_{s'\in\mathrm{Succ}(s)}
F(s\rightarrow s')
\nonumber\\
&=
\sum_{s'\in\mathrm{Succ}(s)}
F(s')
\nonumber\\
&=
\sum_{s'\in\mathrm{Succ}(s)}
\sum_{z\in\mathcal{D}_{\mathrm{term}}(s')}R(z).
\label{eq:app_induction}
\end{align}
Because the searched structure is a tree, the terminal-descendant
sets of different successors are disjoint, and their union equals
$\mathcal{D}_{\mathrm{term}}(s)$. Hence,
\begin{equation}
F(s)
=
\sum_{z\in\mathcal{D}_{\mathrm{term}}(s)}R(z),
\end{equation}
which proves Eq.~\eqref{eq:app_reward_mass}.

For the strategy $a$ leading from $s$ to $s_a$,
Eqs.~\eqref{eq:app_edge_factorization} and
\eqref{eq:app_unique_parent} give
\begin{equation}
\pi_\theta(a\mid s)
=
\frac{F(s_a)}{F(s)}.
\end{equation}
Substituting Eq.~\eqref{eq:app_reward_mass} for $s_a$ and $s$
yields Eq.~\eqref{eq:app_strategy_credit}.
\end{proof}

\paragraph{Connection to AFPO.}
The preceding results characterize the exact flow-consistent
solution. In practice, AFPO parameterizes the state flow as
$F(s)=Q(s)V_\phi(s)$ and minimizes the subpath residual
\begin{equation}
\varepsilon_{m,n}
=
\log\frac{Q(s_n)V_\phi(s_n)}
{Q(s_m)V_\phi(s_m)}
-
\sum_{i=m}^{n-1}\log\pi_\theta(a_i\mid s_i).
\label{eq:app_afpo_residual}
\end{equation}
Rearranging this equality gives
\begin{equation}
\prod_{i=m}^{n-1}\pi_\theta(a_i\mid s_i)
=
\frac{F(s_n)}{F(s_m)}
\exp(-\varepsilon_{m,n}).
\label{eq:app_approximate_credit}
\end{equation}
Therefore, reducing the squared residual
$\varepsilon_{m,n}^{2}$ makes the learned probability of each
strategy subpath approach its corresponding downstream flow ratio.
Since AFPO applies this constraint to every pair $m<n$, a trajectory
of length $T$ provides $O(T^2)$ process-level credit-assignment
relations rather than a single terminal supervision signal.

\subsection{Discussion}
\label{app:credit-discussion}

Theorem~\ref{thm:credit-propagation} formalizes the connection
between long-term affective outcomes and cross-turn strategy
selection. Unlike response-level preference learning, AFPO
constrains how strategy decisions compose across subsequent dialogue
states. Through the MCTS-derived utility $Q(s)$, the learnable
valuation $V_\phi(s)$, and subpath-level flow balance, downstream
support outcomes are converted into supervision for earlier dialogue
prefixes and strategy decisions.

\noindent\textbf{Connection to Experiments.}
The affective-state discrimination analysis in
Figure~\ref{fig:rq2-fig} directly evaluates whether the learned
value signal separates strategy branches with different downstream
outcomes. The turn-wise trajectory analysis in
Figure~\ref{fig:rq4-fig} further examines whether this propagated
process supervision supports coherent long-horizon affective
progression.

\section{Proof: Statistical Grounding of the Search-Derived
Reward Signal}
\label{app:reward-grounding}

This section establishes that the MCTS action value used by AFPO is
a statistically reliable estimate of downstream affective utility.
Specifically, as more rollout evidence is collected, the probability
that $Q(s,a)$ incorrectly orders two support strategies decreases
exponentially.

\paragraph{Setup.}
For a fixed dialogue state--strategy pair $(s,a)$, let
$\widehat{G}_1(s,a),\ldots,\widehat{G}_N(s,a)\in[0,1]$ denote its
backed-up rollout returns. The MCTS action-value estimate is their
empirical mean,
\begin{equation}
Q_N(s,a)
=
\frac{1}{N}\sum_{i=1}^{N}\widehat{G}_i(s,a),
\end{equation}
and let
$U(s,a)=\mathbb{E}[\widehat{G}_i(s,a)]$
denote the expected downstream affective utility. We assume that,
conditional on evaluating a fixed $(s,a)$, the rollout returns are
independent bounded samples with common expectation $U(s,a)$.

\begin{theorem}[Reliability of Search-Derived Preference Signals]
\label{thm:reward-grounding}
Consider two support strategies $a_w$ and $a_l$ at the same state
$s$, with
\begin{equation}
\Delta
=
U(s,a_w)-U(s,a_l)>0.
\end{equation}
Let $N_w$ and $N_l$ denote their respective visit counts. Then
\begin{equation}
\Pr\!\left(
Q_{N_w}(s,a_w)
\leq
Q_{N_l}(s,a_l)
\right)
\leq
\exp\!\left(
-\frac{2\Delta^2}
{1/N_w+1/N_l}
\right).
\label{eq:app_preference_bound}
\end{equation}
Consequently, the probability of assigning the two strategies an
incorrect preference order converges to zero as
$N_w,N_l\rightarrow\infty$.
\end{theorem}

\begin{proof}
Define
\begin{equation}
D
=
Q_{N_w}(s,a_w)-Q_{N_l}(s,a_l).
\end{equation}
Since $\mathbb{E}[D]=\Delta$, an incorrect ordering occurs only when
$D-\mathbb{E}[D]\leq-\Delta$. Writing $D$ as the difference between
two empirical means, the squared ranges of its independent bounded
summands sum to
\begin{equation}
N_w\left(\frac{1}{N_w}\right)^2
+
N_l\left(\frac{1}{N_l}\right)^2
=
\frac{1}{N_w}+\frac{1}{N_l}.
\end{equation}
Applying the one-sided Hoeffding inequality therefore gives
Eq.~\eqref{eq:app_preference_bound}.
\end{proof}

This result provides a statistical basis for constructing AFPO
preference pairs according to
$Q(s,a_w)>Q(s,a_l)$. When two strategies differ in their expected
downstream affective utility, additional rollout evidence makes the
induced preference signal increasingly reliable. The corresponding
ranking objective can therefore transfer search-derived distinctions
between support strategies to the learned valuation model.

The guarantee concerns consistency with the affective utility
defined by the simulator and rewarder; it does not independently
establish the external validity of that reward function. The
affective-state discrimination analysis in
Figure~\ref{fig:rq2-fig} empirically complements this result by
examining whether the learned value signal separates branches with
different downstream outcomes.

\section{Proof: Diversity Guarantee of AFPO}
\label{app:afpo-diversity}

This section builds on the subpath flow property established in
Appendix~\ref{app:credit-propagation} and shows that AFPO induces a
reward-proportional distribution over complete support trajectories,
thereby preserving diversity across multiple successful paths.

\paragraph{Setup.}
Let
$\tau=(s_0,a_0,s_1,a_1,\ldots,s_T)$
denote a complete trajectory in the affective search tree, and let
$R(\tau)\geq 0$ denote its terminal affective reward. Define the
positive-reward trajectory support as
\begin{equation}
\mathcal{T}_{+}
=
\{\tau\mid R(\tau)>0\},
\label{eq:app_positive_trajectory_set}
\end{equation}
and let
\begin{equation}
Z
=
\sum_{\tau\in\mathcal{T}_{+}}R(\tau)
\label{eq:app_diversity_partition}
\end{equation}
denote the total terminal reward mass.

\begin{theorem}[Distribution Matching Property of AFPO]
\label{thm:afpo-diversity-matching}
If the AFPO subpath flow-balance loss is zero and terminal flow is
matched to terminal reward, then the induced trajectory distribution
satisfies
\begin{equation}
\pi_\theta(\tau)
=
\frac{R(\tau)}{Z},
\qquad
\tau\in\mathcal{T}_{+}.
\label{eq:app_diversity_matching}
\end{equation}
Therefore, AFPO samples complete support trajectories with
probability proportional to their terminal affective rewards.
\end{theorem}

\begin{proof}
Applying Lemma~\ref{lem:subpath-flow} to the complete trajectory
from $s_0$ to $s_T$ gives
\begin{equation}
\pi_\theta(\tau)
=
\prod_{i=0}^{T-1}\pi_\theta(a_i\mid s_i)
=
\frac{F(s_T)}{F(s_0)}.
\label{eq:app_complete_flow_ratio}
\end{equation}
By terminal-reward matching,
\begin{equation}
F(s_T)=R(\tau).
\end{equation}
Moreover, flow conservation at the source gives
\begin{equation}
F(s_0)
=
\sum_{\tau'\in\mathcal{T}_{+}}R(\tau')
=
Z.
\end{equation}
Substituting these identities into
Eq.~\eqref{eq:app_complete_flow_ratio} yields
\begin{equation}
\pi_\theta(\tau)
=
\frac{R(\tau)}{Z},
\end{equation}
which proves Eq.~\eqref{eq:app_diversity_matching}.
\end{proof}

This result contrasts with direct reward maximization, whose optimum
may concentrate all probability mass on a single trajectory in
$\arg\max_{\tau}R(\tau)$. In contrast, AFPO assigns non-zero
probability to every searched trajectory with positive terminal
reward.

\begin{theorem}[Entropy Lower Bound]
\label{thm:afpo-diversity-entropy}
Let
\begin{equation}
R_{\max}
=
\max_{\tau\in\mathcal{T}_{+}}R(\tau).
\end{equation}
Under the distribution-matching condition in
Theorem~\ref{thm:afpo-diversity-matching}, the entropy of the
trajectory distribution satisfies
\begin{equation}
H(\pi_\theta)
\geq
\log\!\left(
\frac{Z}{R_{\max}}
\right).
\label{eq:app_diversity_entropy}
\end{equation}
If more than one trajectory has positive reward, this lower bound is
strictly positive.
\end{theorem}

\begin{proof}
Using Eq.~\eqref{eq:app_diversity_matching}, the trajectory entropy
is
\begin{align}
H(\pi_\theta)
&=
-\sum_{\tau\in\mathcal{T}_{+}}
\pi_\theta(\tau)\log\pi_\theta(\tau)
\nonumber\\
&=
-\sum_{\tau\in\mathcal{T}_{+}}
\frac{R(\tau)}{Z}
\log\frac{R(\tau)}{Z}
\nonumber\\
&=
\log Z
-
\mathbb{E}_{\tau\sim\pi_\theta}
[\log R(\tau)].
\label{eq:app_entropy_expansion}
\end{align}
Since $R(\tau)\leq R_{\max}$ for every
$\tau\in\mathcal{T}_{+}$,
\begin{equation}
\mathbb{E}_{\tau\sim\pi_\theta}
[\log R(\tau)]
\leq
\log R_{\max}.
\end{equation}
Therefore,
\begin{equation}
H(\pi_\theta)
\geq
\log Z-\log R_{\max}
=
\log\!\left(
\frac{Z}{R_{\max}}
\right),
\end{equation}
which proves Eq.~\eqref{eq:app_diversity_entropy}. If multiple
trajectories have positive reward, then $Z>R_{\max}$ and the lower
bound is positive.
\end{proof}

These results show that AFPO distributes probability mass across
multiple successful support trajectories rather than collapsing
onto only one maximum-reward path. Because each trajectory contains
a sequence of support strategies, this property also provides a
theoretical basis for the broad strategy coverage evaluated in
Figure~\ref{fig:rq5-fig}.

\section{Detailed Derivation of the Subpath Flow-Balance Constraint}
\label{app:subpath-flow-derivation}

This section provides the intermediate steps omitted from the main
text. We first derive the flow-balance relation for an arbitrary
subpath. We then substitute the affective flow parameterization.
Finally, we obtain the log-space training objective.

We consider only valid subpaths with positive state flows and positive
forward probabilities. This condition ensures that cancellation and
logarithmic operations are well-defined.

\paragraph{From Edge-Wise Matching to Subpath Balance.}

Consider a continuous subpath in the affective search tree:
\begin{equation}
s_m
\xrightarrow{a_m}
s_{m+1}
\xrightarrow{a_{m+1}}
\cdots
\xrightarrow{a_{n-1}}
s_n,
\qquad m<n,
\label{eq:app_subpath}
\end{equation}
where $s_i$ denotes the dialogue-prefix state at turn $i$.
The variable $a_i$ denotes the support strategy selected at $s_i$.
Let $F(s_i)$ denote the flow of state $s_i$.
Let $\pi_\theta(a_i\mid s_i)$ denote the forward policy.
Let $\pi_B(s_i\mid s_{i+1})$ denote the backward policy.

For each visited transition
$s_i\xrightarrow{a_i}s_{i+1}$, the edge flow can be written in the
forward direction as
\begin{equation}
F(s_i\rightarrow s_{i+1})
=
F(s_i)\pi_\theta(a_i\mid s_i).
\label{eq:app_forward_edge_flow}
\end{equation}
The same edge flow can be written in the backward direction as
\begin{equation}
F(s_i\rightarrow s_{i+1})
=
F(s_{i+1})\pi_B(s_i\mid s_{i+1}).
\label{eq:app_backward_edge_flow}
\end{equation}
These two expressions describe the same edge. Therefore,
\begin{equation}
F(s_i)\pi_\theta(a_i\mid s_i)
=
F(s_{i+1})\pi_B(s_i\mid s_{i+1}).
\label{eq:app_edge_matching}
\end{equation}

Equation~\eqref{eq:app_edge_matching} holds for every transition on
the subpath. We multiply it over $i=m,\ldots,n-1$:
\begin{align}
&
\left(
\prod_{i=m}^{n-1}F(s_i)
\right)
\left(
\prod_{i=m}^{n-1}\pi_\theta(a_i\mid s_i)
\right)
\nonumber\\
&\qquad =
\left(
\prod_{i=m}^{n-1}F(s_{i+1})
\right)
\left(
\prod_{i=m}^{n-1}\pi_B(s_i\mid s_{i+1})
\right).
\label{eq:app_product_matching}
\end{align}

The state-flow products can be expanded explicitly:
\begin{align}
&
F(s_m)
\left[
\prod_{j=m+1}^{n-1}F(s_j)
\right]
\prod_{i=m}^{n-1}\pi_\theta(a_i\mid s_i)
\nonumber\\
&\qquad =
F(s_n)
\left[
\prod_{j=m+1}^{n-1}F(s_j)
\right]
\prod_{i=m}^{n-1}\pi_B(s_i\mid s_{i+1}).
\label{eq:app_expanded_matching}
\end{align}

Both sides contain the same intermediate-flow term:
\begin{equation}
\prod_{j=m+1}^{n-1}F(s_j).
\end{equation}
This common term can be cancelled. We then obtain
\begin{equation}
F(s_m)
\prod_{i=m}^{n-1}\pi_\theta(a_i\mid s_i)
=
F(s_n)
\prod_{i=m}^{n-1}\pi_B(s_i\mid s_{i+1}).
\label{eq:app_general_subpath_balance}
\end{equation}

This cancellation has a telescoping structure. All internal state
flows disappear. Only the flows at the two endpoints remain.

AFlow uses a tree of path-specific dialogue prefixes. Different
search trajectories remain separate. They are not merged into the
same successor node. Therefore, every non-root node has a unique
predecessor. For every recorded transition,
\begin{equation}
\pi_B(s_i\mid s_{i+1})=1.
\label{eq:app_unique_predecessor}
\end{equation}
Equation~\eqref{eq:app_general_subpath_balance} therefore becomes
\begin{equation}
F(s_n)
=
F(s_m)
\prod_{i=m}^{n-1}\pi_\theta(a_i\mid s_i).
\label{eq:app_tree_subpath_balance}
\end{equation}
Equivalently,
\begin{equation}
\prod_{i=m}^{n-1}\pi_\theta(a_i\mid s_i)
=
\frac{F(s_n)}{F(s_m)}.
\label{eq:app_flow_ratio}
\end{equation}

\paragraph{Substitution of the Affective Flow Parameterization.}

AFPO parameterizes the state flow as
\begin{equation}
F(s)=Q(s)V_\phi(s),
\label{eq:app_flow_parameterization}
\end{equation}
where $Q(s)$ is the state-utility estimate obtained from MCTS.
The function $V_\phi(s)$ is a learnable non-negative state
valuation.

Substituting Eq.~\eqref{eq:app_flow_parameterization} into
Eq.~\eqref{eq:app_tree_subpath_balance} gives
\begin{equation}
Q(s_n)V_\phi(s_n)
=
Q(s_m)V_\phi(s_m)
\prod_{i=m}^{n-1}\pi_\theta(a_i\mid s_i).
\label{eq:app_parameterized_subpath_balance}
\end{equation}

This equation connects three quantities. They are the search-derived
state utility, the learned state valuation, and the policy
probability assigned to the strategy sequence.

\paragraph{Log-Space Form and Training Objective.}

A direct product of many probabilities can cause numerical
underflow. AFPO therefore constructs the objective in log space.
Taking the natural logarithm of both sides of
Eq.~\eqref{eq:app_parameterized_subpath_balance} gives
\begin{align}
&
\log\!\big(Q(s_n)V_\phi(s_n)\big)
-
\log\!\big(Q(s_m)V_\phi(s_m)\big)
\nonumber\\
&\qquad =
\sum_{i=m}^{n-1}
\log\pi_\theta(a_i\mid s_i).
\label{eq:app_log_subpath_balance}
\end{align}
Here, we use the identity
$\log\prod_i x_i=\sum_i\log x_i$.

We define the change in log flow between the two endpoints as
\begin{equation}
\Delta F_{m,n}
=
\log\!\big(Q(s_n)V_\phi(s_n)\big)
-
\log\!\big(Q(s_m)V_\phi(s_m)\big).
\label{eq:app_delta_flow}
\end{equation}
We also define the accumulated strategy log-probability as
\begin{equation}
\Delta\pi_{m,n}
=
\sum_{i=m}^{n-1}
\log\pi_\theta(a_i\mid s_i).
\label{eq:app_delta_policy}
\end{equation}

Exact subpath flow balance requires
\begin{equation}
\Delta F_{m,n}
=
\Delta\pi_{m,n}.
\label{eq:app_exact_subpath_balance}
\end{equation}
During training, this equality is not satisfied exactly. AFPO
therefore minimizes the squared residual:
\begin{equation}
\mathcal{L}_{\mathrm{policy}}
=
\sum_{\tau\in\mathcal{T}}
\sum_{0\leq m<n\leq T_\tau}
\left(
\Delta F_{m,n}
-
\Delta\pi_{m,n}
\right)^2,
\label{eq:app_policy_loss}
\end{equation}
where $\mathcal{T}$ is the set of searched trajectories.
The variable $T_\tau$ is the number of strategy transitions in
trajectory $\tau$.

A trajectory with $T$ transitions contains
\begin{equation}
\frac{T(T+1)}{2}=O(T^2)
\end{equation}
continuous subpaths. These subpaths cover different temporal scales.
They include single transitions, local strategy sequences, and the
complete trajectory.

\paragraph{Intuitive Interpretation.}

Equation~\eqref{eq:app_flow_ratio} provides a direct interpretation.
Consider several candidate subpaths that start from the same state
$s_m$. Their starting flow is shared. A subpath that reaches a state
with a larger endpoint flow should receive a larger joint
probability. A subpath that reaches a lower-flow state should receive
a smaller joint probability.

This relation provides long-horizon credit to early strategies. An
early information-seeking strategy may not produce the highest
immediate reward. However, it may reveal the user's main concern. It
may then enable more effective emotional reflection or problem
solving. In this case, the early strategy receives positive credit
through the larger downstream flow.

The constraint does not require adjacent strategy labels to be
similar. It also does not enforce a smooth strategy sequence. A
sudden transition is allowed when the user's state changes. For
example, the model may move from emotional comfort to information
exploration or action planning. AFPO favors this transition when it
leads to a more valuable downstream state. The constraint therefore
models consistency between strategy probabilities and long-term
support quality. It does not model surface continuity between
adjacent strategy labels.

\section{Related Work}
\label{app:related}

\subsection{Emotional Support Conversation}

Emotional Support Conversation (ESC) studies multi-turn interactions
in which a supporter responds to a distressed seeker with empathy,
information, and context-appropriate supportive actions.
Liu et al.~\cite{liu2021towards} formalize the task and release  
ESConv, establishing a standard benchmark for strategy-aware
supportive dialogue. Subsequent work improves strategy conditioning
and interpretability by explicitly modeling the supportive action
associated with each response.

MISC introduces mixed-strategy awareness and external commonsense
knowledge to guide response generation~\cite{tu2022misc},
whereas ESCoT exposes strategy-related reasoning to improve
interpretability~\cite{zhang2024escot}. Other approaches strengthen
the connection between dialogue context and supportive action through
lookahead planning~\cite{cheng2022improving},
persona-aware modeling~\cite{cheng2023pal}, and knowledge-enhanced
mixed-initiative generation~\cite{deng2023knowledge}. These methods
provide useful task-specific inductive biases, but their supervision
is generally attached to individual turns, strategy labels, or
auxiliary modules. As a result, they do not explicitly propagate
downstream support outcomes back to earlier dialogue prefixes.

\subsection{LLM-Based Supporters}

Recent work increasingly adopts large language models as ESC
supporters because of their strong open-domain generation and
instruction-following capabilities. LLM-generated self-chats and
synthetic dialogues have been used to expand training data and improve
supportive behavior~\cite{zheng2024self}. At the same time,
empirical analyses reveal systematic preference biases and
misalignment in LLM supporters, motivating targeted prompting and
mitigation strategies~\cite{kang2024can}.

Beyond single-agent prompting, simulation and collaboration have been
used to introduce richer interaction structure. Socialized simulation
provides multi-agent environments for studying emotional support
behavior~\cite{chen2025socialsim}, while MultiAgentESC coordinates
multiple LLM agents to improve supportive response generation
~\cite{xu2025multiagentesc}. Evaluation has also become more
systematic, with ESC-Eval defining protocols for assessing LLM-based
supporters across multiple dimensions~\cite{zhao2024esc}. Although
these approaches improve fluency, scalability, and evaluation
coverage, their control signals are often prompt-level or
dialogue-level. They therefore provide limited supervision for how a
support strategy should change from one dialogue prefix to the next.

\subsection{Preference Alignment}

Preference-based alignment offers a flexible alternative to explicit
strategy supervision. RLHF~\cite{ouyang2022training} and
DPO~\cite{rafailov2023direct} optimize language models from human or
model-generated comparisons, and later methods such as
SimPO~\cite{meng2024simpo} simplify the optimization procedure. In
ESC, preference signals have been used to align responses with
supportive judgments and to refine strategy-related behavior,
including chain-style strategy optimization~\cite{zhao2025chain}.

Most preference-based methods, however, compare complete responses or
entire dialogues. Such supervision can identify which output is
better, but it does not directly explain which intermediate strategy
decision produced the improvement. This limitation is particularly
important in multi-turn support, where the effect of an earlier
clarification, validation, or suggestion may become apparent only
after several subsequent turns. The resulting credit-assignment
problem motivates objectives that operate over dialogue prefixes and
subpaths rather than isolated response pairs.

\subsection{Process-Aware Strategy Refinement}

A complementary line of work introduces additional process structure
through planning, affect dynamics, reinforcement learning, or
explicit strategy control. Lookahead planning anticipates future
dialogue evolution when selecting the next supportive action
~\cite{cheng2022improving}. EmoDynamiX models mixed emotions
and discourse dynamics to provide richer signals for strategy
prediction and response generation~\cite{wan2025emodynamix}.
Reinforcement-learning objectives have also been explored for
multi-turn ESC, including positive emotion elicitation as an explicit
optimization target~\cite{zhou2023facilitating}.

Other approaches steer LLMs over long emotional-support interactions
~\cite{madani2025steering} or combine dynamic knowledge filtering
with persona extraction~\cite{hao2025enhancing}. Compared with
response-level preference learning, these methods provide stronger
process guidance and better account for dialogue evolution.
Nevertheless, they typically depend on task-specific planners,
auxiliary dynamics modules, or scalar sequence-level objectives. They
do not impose a single consistency relation that links the values of
intermediate prefixes, strategy transitions, and downstream outcomes
throughout the dialogue.

\subsection{Positioning of AFlow}

AFlow addresses this gap by treating multi-turn emotional support as
a non-negative affective utility flow over dialogue states and
strategy transitions. The framework uses search-constructed
trajectories to estimate the utility of intermediate states and
actions, and then applies flow-balance principles
~\cite{bengio2023gflownet} to propagate downstream preference signals
across dialogue subpaths.

This formulation differs from prior ESC systems in two respects.
First, supervision is attached to dialogue prefixes and strategy
transitions rather than only to complete responses or terminal
outcomes. Second, the subpath objective constrains how local strategy
decisions compose over multiple turns, providing a structurally
consistent form of long-horizon credit assignment. AFlow is therefore
best viewed as complementary to existing strategy-aware and
LLM-based ESC models: it supplies a process-level optimization
principle that can be applied on top of a language-model supporter
without requiring a separate hand-designed dialogue planner.

\section{Simulation Environment and Prompts}
\label{app:environment}

This appendix specifies the role-separated simulation environment,
the prompts used for dialogue expansion and turn-level scoring, the
affective evaluation rubric, and the support-strategy inventory.
The design focuses on three objectives: ensuring controllable
turn-level generation, suppressing meta-level artifacts that may
confound strategy comparison, and grounding generated content in the
dialogue history.

At each dialogue step, the Supporter generates a response conditioned
on a selected support strategy. The Rewarder evaluates the latest
Supporter utterance, after which the Seeker generates the next
persona-grounded utterance or an explicit termination token.
Although all roles receive the dialogue history, their private
instructions and output spaces remain separated.

\subsection{Role Prompts and Constraints}
\label{app:B1_prompt_cards}

The following cards present the operational prompts used for the
Seeker, Supporter, and Rewarder. Strict length, content, and output
constraints make candidate responses comparable across strategies
and reduce variation caused by verbosity, formatting, or hallucinated
context.

\noindent\textbf{Seeker role.}
The Seeker prompt instantiates a persona-grounded user simulator for
role-separated dialogue expansion. It requires one short user
utterance per turn, prohibits system-related or evaluative language,
and uses explicit termination tokens to prevent unbounded rollouts.

\begin{figure}[t]
\centering
\begin{aflowcard}{Seeker Prompt: Persona-Grounded User Simulation}

\aflowrole{System Role}
You are a simulated human help-seeker (\emph{Seeker}) in a
psychological-support dialogue evaluation.

\aflowsep
\aflowrole{Core Requirement}
Simulate a specific persona grounded in the provided data rather than
behaving as a generic assistant. Internalize the persona's emotional
state, background, and linguistic style, but do not copy the private
Memory verbatim.

\aflowsep
\aflowrole{Inputs}
\begin{itemize}[leftmargin=*,nosep]
  \item Context Scene: \ph{scene}
  \item Core Background (private): \ph{description}
  \item Memory (private): \ph{memory}
  \item Dialogue History: \ph{chat\_history}
\end{itemize}

\aflowsep
\aflowrole{Persona Grounding}
Reveal background information gradually, introducing at most one new
detail per turn. Infer relevant emotions and facts from the private
Memory, but do not quote or closely paraphrase it. Never reproduce a
complete sentence or any contiguous span of eight or more words from
the Memory.

\aflowsep
\aflowrole{Output Rules}
\begin{itemize}[leftmargin=*,nosep]
  \item Output only the next Seeker utterance or the exact token
        \tokend.
  \item Maximum length: \textbf{20 space-separated words}.
  \item Do not add prefixes, role labels, explanations, or meta text.
  \item Do not mention AI, system prompts, evaluation, policies, or
        role names.
\end{itemize}

\aflowsep
\aflowrole{Style and Behavior}
Use informal and colloquial language. The emotional state may
fluctuate, and the Seeker should maintain realistic guardedness. If
misunderstood, the Seeker may push back, deflect, become brief, or
disengage.

\aflowsep
\aflowrole{Conversation-Length Control}
Infer the number of Seeker turns from \ph{chat\_history}. The core
situation or problem should become clear by approximately Seeker
turns 8--10. After turn 10, do not introduce major new background
details.

\aflowsep
\aflowrole{Ending Rules}
If the Seeker feels sufficiently relieved, calm, or understood,
produce a brief natural sign-off followed by \tokenok. Use \tokend
only when the Seeker disengages, the interaction becomes implausible,
the Supporter appears generic or robotic, or no believable
continuation remains.

\aflowsep
\aflowrole{Task}
Produce the next natural Seeker response based on the dialogue
history while satisfying all constraints.

\end{aflowcard}
\label{fig:seeker-prompt}
\end{figure}

\noindent\textbf{Supporter role.}
The Supporter prompt enforces strategy-conditioned generation by
providing a target strategy and its definition as private guidance.
It requires one natural response that realizes the strategy within a
single turn. Length constraints and prohibitions on meta language,
clinical jargon, and formal customer-service phrasing reduce
stylistic confounds across candidate strategies.

\begin{figure}[t]
\centering
\begin{aflowcard}{Supporter Prompt: Strategy-Conditioned Response}

\aflowrole{System Role}
You are a simulated human emotional supporter (\emph{Supporter}) in
a face-to-face psychological-support dialogue evaluation.

\aflowsep
\aflowrole{Core Requirement}
Follow the assigned strategy precisely while sounding natural,
spoken, and emotionally responsive rather than robotic or clinical.
You are not an AI assistant.

\aflowsep
\aflowrole{Inputs}
\begin{itemize}[leftmargin=*,nosep]
  \item Target Strategy: \ph{strategy}
  \item Strategy Definition (private guidance):
        \ph{strategy\_detail}
  \item Dialogue History: \ph{chat\_history}
\end{itemize}

\aflowsep
\aflowrole{Strategy Execution}
Internalize the strategy definition but do not name, quote, recite,
or explain it. Implement the assigned strategy in the current turn.
Prefer emotional reflection or validation followed by one small
supportive move when appropriate.

\aflowsep
\aflowrole{Output Rules}
\begin{itemize}[leftmargin=*,nosep]
  \item Output only the next Supporter utterance.
  \item Maximum length: \textbf{25 space-separated words}.
  \item Contractions count as one word.
  \item Use at most one question mark.
  \item Do not add labels, headers, quotation marks, or meta text.
\end{itemize}

\aflowsep
\aflowrole{Prohibitions}
\begin{itemize}[leftmargin=*,nosep]
  \item Do not mention AI, prompts, evaluation, policies, or strategy
        labels.
  \item Do not use clinical or diagnostic jargon.
  \item Do not use formal customer-service phrasing, such as
        ``I apologize for the inconvenience.''
\end{itemize}

\aflowsep
\aflowrole{Content Rules}
Use only information stated or clearly implied in
\ph{chat\_history}. Do not invent facts about the Seeker or their
circumstances. Implement the target strategy rather than describing
it, and keep the response grounded in the Seeker's current needs.

\aflowsep
\aflowrole{Task}
Produce the next Supporter response that best realizes
\ph{strategy} under all stated constraints.

\end{aflowcard}
\label{fig:supporter-prompt}
\end{figure}

\noindent\textbf{Rewarder role.}
The Rewarder prompt specifies an impartial turn-level evaluator. It
scores only the latest Supporter utterance according to a fixed set
of criteria. Integer-only outputs without free-form explanations
reduce verbosity and positional biases and provide structured reward
signals for search and downstream optimization.

\begin{figure}[t]
\centering
\begin{aflowcard}{Rewarder Prompt: Turn-Level Affective Scoring}

\aflowrole{System Role}
You are an expert psychologist serving as an impartial judge for a
psychological-support dialogue evaluation.

\aflowsep
\aflowrole{Core Requirement}
Score only the Supporter's most recent utterance according to the
provided criteria. Judge psychological effectiveness rather than
verbosity, and minimize positional and length biases.

\aflowsep
\aflowrole{Inputs}
\begin{itemize}[leftmargin=*,nosep]
  \item Evaluation Criteria: \ph{evaluation\_criteria}
  \item Dialogue History: \ph{chat\_history}
\end{itemize}

\aflowsep
\aflowrole{Evaluation Target}
Evaluate only the most recent Supporter utterance. Earlier dialogue
turns may be used solely as context for understanding the Seeker's
emotional state, situation, and current needs.

\aflowsep
\aflowrole{Scoring Rules}
\begin{itemize}[leftmargin=*,nosep]
  \item Assign one integer score from 1 to 5 for each criterion.
  \item A score of 1 is the lowest and 5 is the highest.
  \item Preserve the order of \ph{evaluation\_criteria}.
  \item Do not reward or penalize response length.
  \item Output no explanations or additional text.
\end{itemize}

\aflowsep
\aflowrole{Output Format}
Return one line for each criterion in exactly the following format:
\begin{center}
\texttt{<Question number>: <Score>}
\end{center}

\aflowsep
\aflowrole{Task}
Output the scores only, in the required order and format.

\end{aflowcard}
\label{fig:rewarder-prompt}
\end{figure}

\subsection{Reward Dimensions and Scoring Rubric}
\label{app:B2_eval_rubrics}

The Rewarder evaluates the latest Supporter utterance along four
dimensions: empathy, information quality, humanoid naturalness, and
strategic efficacy. Each dimension includes an operational
definition, a brief theoretical grounding, and discrete behavioral
anchors for scores from 1 to 5.

The four dimensions separately measure emotional attunement,
informational usefulness, conversational realism, and the
appropriateness of the selected support strategy. Their use provides
consistent turn-level assessment across dialogue states, strategies,
and model configurations.

\begin{figure*}[!p]
\centering

\begingroup
\small
\setlength{\parindent}{0pt}
\setlength{\parskip}{0pt}

\begin{aflowcard}{Empathy Scoring Rubric}

\aflowrole{Definition}
The ability to recognize, understand, and validate the Seeker's
emotional state and underlying needs.

\smallskip
\aflowrole{Theoretical Basis}
Rogers' Person-Centered Therapy, particularly empathic understanding
and unconditional positive regard.

\medskip

\setlength{\tabcolsep}{2pt}
\renewcommand{\arraystretch}{1.08}
\begin{tabularx}{\linewidth}{@{}
  >{\centering\arraybackslash\bfseries}p{0.045\linewidth}
  @{\hspace{6pt}}
  >{\raggedright\arraybackslash}X
@{}}
5 &
\textbf{Profound empathy.}
Accurately recognizes and validates both explicit and implied
emotions, reflects the underlying source of distress, and maintains
a consistently warm and non-judgmental tone. \\[2.5pt]

4 &
\textbf{High empathy.}
Correctly identifies and validates the primary emotion, demonstrates
clear perspective-taking, and establishes a supportive emotional
stance. \\[2.5pt]

3 &
\textbf{Moderate empathy.}
Acknowledges the Seeker's emotion but does so superficially or
generically, with limited contextual grounding or a mildly mismatched
tone. \\[2.5pt]

2 &
\textbf{Low empathy.}
Misreads or overlooks the emotional need, relies on toxic positivity,
or prioritizes factual advice over emotional acknowledgment. \\[2.5pt]

1 &
\textbf{Absence of empathy.}
Responds with apathy, judgment, or invalidation. The response is cold,
dismissive, or potentially harmful. \\
\end{tabularx}

\end{aflowcard}

\vspace{6pt}

\begin{aflowcard}{Information Quality Scoring Rubric}

\aflowrole{Definition}
The relevance, accuracy, specificity, and practical value of the
information or guidance provided to the Seeker.

\smallskip
\aflowrole{Theoretical Basis}
Psychoeducation principles and effectiveness criteria from
Problem-Solving Therapy.

\medskip

\setlength{\tabcolsep}{2pt}
\renewcommand{\arraystretch}{1.08}
\begin{tabularx}{\linewidth}{@{}
  >{\centering\arraybackslash\bfseries}p{0.045\linewidth}
  @{\hspace{6pt}}
  >{\raggedright\arraybackslash}X
@{}}
5 &
\textbf{Optimal utility.}
Provides accurate, specific, and highly actionable guidance tailored
to the Seeker's situation, including useful insights or concrete next
steps that address the underlying problem. \\[2.5pt]

4 &
\textbf{High utility.}
Provides relevant and practical guidance that is well matched to the
situation, with only minor gaps in depth, specificity, or
personalization. \\[2.5pt]

3 &
\textbf{Moderate utility.}
Provides safe and broadly appropriate guidance, but the content is
generic, only partly actionable, or weakly grounded in the Seeker's
specific context. \\[2.5pt]

2 &
\textbf{Low utility.}
Provides vague, confusing, or largely irrelevant guidance that offers
little practical value in the current situation. \\[2.5pt]

1 &
\textbf{Counterproductive or harmful.}
Provides inaccurate, unsafe, or irrelevant information that may
mislead the Seeker or worsen the situation. \\
\end{tabularx}

\end{aflowcard}

\label{fig:rubric-pair-one}

\vspace{6pt}

\begin{aflowcard}{Humanoid Naturalness Scoring Rubric}

\aflowrole{Definition}
The degree to which the response resembles natural human conversation
while avoiding robotic, repetitive, or template-like patterns.

\smallskip
\aflowrole{Theoretical Basis}
Turing-style criteria emphasizing coherence, linguistic variation,
persona consistency, and the absence of conspicuous machine-generated
artifacts.

\medskip

\setlength{\tabcolsep}{2pt}
\renewcommand{\arraystretch}{1.08}
\begin{tabularx}{\linewidth}{@{}
  >{\centering\arraybackslash\bfseries}p{0.045\linewidth}
  @{\hspace{6pt}}
  >{\raggedright\arraybackslash}X
@{}}
5 &
\textbf{Indistinguishable from human.}
Maintains seamless conversational flow, natural variation, and
context-sensitive phrasing without noticeable template markers,
robotic formatting, or stylistic artifacts. \\[2.5pt]

4 &
\textbf{Highly natural.}
Produces fluid and engaging conversation with consistent tone and
persona. Minor formality may appear but does not disrupt the
interaction. \\[2.5pt]

3 &
\textbf{Detectable model traits.}
Remains coherent and appropriate but sounds overly formal, structured,
or textbook-like, with occasional repetition or predictable phrasing. \\[2.5pt]

2 &
\textbf{Robot-like.}
Produces stiff, repetitive, or visibly templated responses with a
clinical, monotonous, or poorly adapted tone. \\[2.5pt]

1 &
\textbf{Machine-like output.}
Produces disjointed, unnatural, or artifact-laden text that clearly
breaks conversational realism. \\
\end{tabularx}

\end{aflowcard}

\vspace{6pt}

\begin{aflowcard}{Strategic Efficacy Scoring Rubric}

\aflowrole{Definition}
The appropriateness and effectiveness of the selected support strategy
for the Seeker's emotional state, needs, and current dialogue stage.

\smallskip
\aflowrole{Theoretical Basis}
Phase models of psychotherapy, emphasizing the alignment of
interventions with exploration, insight, and action stages.

\medskip

\setlength{\tabcolsep}{2pt}
\renewcommand{\arraystretch}{1.08}
\begin{tabularx}{\linewidth}{@{}
  >{\centering\arraybackslash\bfseries}p{0.045\linewidth}
  @{\hspace{6pt}}
  >{\raggedright\arraybackslash}X
@{}}
5 &
\textbf{Masterful execution.}
Selects and executes a well-timed strategy that meaningfully advances
the interaction toward relief, insight, or constructive action.
Multiple supportive moves are combined flexibly when appropriate. \\[2.5pt]

4 &
\textbf{Effective execution.}
Uses an appropriate and correctly implemented strategy that contributes
positively to the interaction, with only minor limitations in nuance
or timing. \\[2.5pt]

3 &
\textbf{Basic execution.}
Uses a recognizable strategy that broadly fits the context, but the
execution is generic, awkward, or insufficiently developed. \\[2.5pt]

2 &
\textbf{Misaligned strategy.}
Uses a strategy that does not fit the Seeker's current needs or
dialogue stage and therefore feels premature, forced, or poorly timed. \\[2.5pt]

1 &
\textbf{Counterproductive strategy.}
Uses an unsuitable approach that escalates distress, undermines
rapport, or obstructs effective emotional support. \\
\end{tabularx}

\end{aflowcard}

\label{fig:rubric-pair-two}

\endgroup
\end{figure*}

Taken together, these dimensions distinguish emotional
understanding, practical usefulness, conversational realism, and
strategy appropriateness. Their weighted aggregation provides the
turn-level affective reward used during trajectory search, while the
individual scores remain available for diagnostic analysis.

\subsection{Strategy Action Set $\mathcal{A}$}
\label{app:B3_strategies}

The Supporter action space contains eight strategy categories
covering emotional exploration, affective validation, information
provision, and action-oriented support. Each strategy is defined as
a concrete single-turn communicative action so that different MCTS
branches correspond to comparable strategy decisions.

The abbreviations below are used consistently in the search
procedure, figures, algorithms, annotations, and
strategy-distribution analyses.

\begin{table*}[t]
\centering

\caption{Supporter strategy set used in this work.}
\label{tab:strategy-inventory}

\begingroup
\small
\setlength{\parindent}{0pt}
\setlength{\parskip}{0pt}

\begin{minipage}[t]{0.485\textwidth}
\vspace{0pt}
\begin{aflowcard}{Reflective Statements (RS)}

\aflowrole{Definition}
Restate or paraphrase the Seeker's central message to demonstrate
understanding. Preserve the intended meaning without adding advice,
interpretations, or new facts.

\end{aflowcard}
\end{minipage}
\hfill
\begin{minipage}[t]{0.485\textwidth}
\vspace{0pt}
\begin{aflowcard}{Clarification (Cla)}

\aflowrole{Definition}
Ask one brief and focused question to resolve an ambiguity about the
Seeker's situation, goal, or need. Keep the tone gentle and avoid
combining several questions.

\end{aflowcard}
\end{minipage}

\vspace{6pt}

\begin{minipage}[t]{0.485\textwidth}
\vspace{0pt}
\begin{aflowcard}{Emotional Validation (EV)}

\aflowrole{Definition}
Acknowledge that the Seeker's feelings are understandable in context.
Validate the emotion without endorsing inaccurate beliefs or judging
the Seeker's choices.

\end{aflowcard}
\end{minipage}
\hfill
\begin{minipage}[t]{0.485\textwidth}
\vspace{0pt}
\begin{aflowcard}{Empathetic Statements (ES)}

\aflowrole{Definition}
Express warm emotional understanding by reflecting what the experience
may feel like for the Seeker. Prioritize emotional presence over
immediate problem solving.

\end{aflowcard}
\end{minipage}

\vspace{6pt}

\begin{minipage}[t]{0.485\textwidth}
\vspace{0pt}
\begin{aflowcard}{Affirmation (Aff)}

\aflowrole{Definition}
Recognize a specific strength, effort, or value demonstrated by the
Seeker. Reinforce agency and self-efficacy without exaggeration or
empty praise.

\end{aflowcard}
\end{minipage}
\hfill
\begin{minipage}[t]{0.485\textwidth}
\vspace{0pt}
\begin{aflowcard}{Collaborative Planning (CP)}

\aflowrole{Definition}
Work with the Seeker to identify a small and feasible next step that
matches their readiness. Encourage joint planning while preserving
autonomy and avoiding prescriptive language.

\end{aflowcard}
\end{minipage}

\vspace{6pt}

\begin{minipage}[t]{0.485\textwidth}
\vspace{0pt}
\begin{aflowcard}{Suggest Options (SO)}

\aflowrole{Definition}
Offer a small set of concrete and realistic options for the Seeker to
consider. Keep the choices low-pressure and limited enough to avoid
overwhelming them.

\end{aflowcard}
\end{minipage}
\hfill
\begin{minipage}[t]{0.485\textwidth}
\vspace{0pt}
\begin{aflowcard}{Share Information (SI)}

\aflowrole{Definition}
Provide concise and accurate information or a simple explanatory
concept relevant to the current concern. Use accessible, non-clinical
language and connect it directly to the dialogue context.

\end{aflowcard}
\end{minipage}

\endgroup
\end{table*}

%

\section{Supplementary Experimental Details}
\label{app:experimental_details}

This appendix provides the experimental details that are not repeated
in the main text, including seed preprocessing, conversion of MCTS
search traces into AFPO supervision, comparison controls, and run-level
reproducibility records.

\subsection{Data Preparation and Isolation}
\label{app:data}
\label{app:data:seed}
\label{app:data:fields}
\label{app:data:split}

\paragraph{Seed sampling.}
We sample 1{,}000 dialogues from ExTES and 500 from ESConv.
Sampling is stratified by the scenario or emotion annotations provided
by each dataset. We first ensure coverage of the available labels and
then allocate the remaining budget to approximately preserve the
original label-frequency distribution.

\paragraph{Dialogue normalization.}
Each seed dialogue is converted into a common record containing
\texttt{scene}, \texttt{description}, and \texttt{memory}.
The \texttt{scene} field summarizes the situation presented to the
dialogue agents. The \texttt{description} field records the central
background or problem statement and is treated as private context when
required. The optional \texttt{memory} field contains only persistent
constraints or preferences explicitly stated in the source dialogue;
otherwise, it is left empty. The same schema is used for training-time
and evaluation-time expansion.

\paragraph{Training--evaluation isolation.}
The seed pool is partitioned before any MCTS expansion. Training trees
are generated only from training seeds. Held-out seeds are expanded
with frozen model parameters, and their states, trajectories, and
search statistics are never used for optimization. This ordering
prevents leakage through either the original dialogues or their
search-generated continuations.

\subsection{Trajectory Extraction and AFPO Supervision}
\label{app:data:mcts}
\label{app:data:signals}
\label{app:data:stats}
\label{app:data:iter}
\label{app:impl_mcts_data}

The search procedure and its default hyperparameters are specified in
\S\ref{subsec:mcts} and Appendix~\ref{app:depth}. Here we describe how
recorded search traces are converted into training instances. One tree
is constructed for each seed dialogue. A trajectory is retained only
when it contains more than three nodes. For each retained trajectory,
all valid contiguous subpaths are enumerated for the AFPO flow-balance
objective. Visited states additionally yield action pairs ordered by
their search-derived values $Q(s,a)$.

Each tree contributes approximately 10 retained trajectories. Their
average length is 8.25 turns, producing approximately 38 subpath
instances per trajectory. Across the expanded dialogues, the mean
dialogue length is 8.29 turns (median 6), and the mean utterance length
is 24.9 tokens under the shared tokenization and segmentation
protocol.

The tree corpus is refreshed every 10 training iterations using the
latest policy and value models. The seed partition, search budget,
termination rule, and trajectory filter remain fixed across refreshes.
Consequently, the supervision distribution can evolve with the model
without changing the underlying data split.

\subsection{Comparison and Evaluation Controls}
\label{app:metrics}
\label{app:baselines}

All systems are evaluated on matched dialogue prefixes obtained from
the held-out expansion. Comparable systems use the same preprocessing,
tokenization, metric implementations, maximum response length, and
stopping rules. No additional filtering, reranking, or post-hoc
rewriting is introduced unless it is part of the original method.

Strategy Accuracy and macro-F1 are reported in the unified
eight-category strategy space used in the main paper. When a system
uses a different taxonomy, its labels are converted through a fixed
many-to-one mapping defined before evaluation. The mapping is applied
uniformly and is not modified after observing model outputs.

Perplexity is reported only when token-level likelihoods are available
and comparable under the fixed scorer. Strategy metrics are omitted
for systems that do not produce strategy predictions. These cases are
marked directly in the main results table.

Pairwise comparisons follow the held-out, matched-pair protocol in
Appendix~\ref{app:human_eval}. System identities are hidden, candidate
order is randomized independently for each instance, and the same
rubric is used across compared systems. The annotation form, judge
prompt, parsing procedure, and Win/Tie/Lose aggregation are provided
there and are not repeated here.

\subsection{Implementation and Run Records}
\label{app:exp_details}
\label{app:impl_arch}
\label{app:impl_afpo_train}
\label{app:impl_infer}
\label{app:impl_compute}

The model architecture, AFPO objective, and inference procedure are
defined in the main method. Across backbone and simulation-environment
comparisons, we keep the training-iteration budget, tree-refresh
interval, stopping criterion, and evaluation contexts fixed. The
performance--cost analysis of the MCTS configuration is reported in
Appendix~\ref{app:depth}.

Training and offline search use a shared pool of 8 NVIDIA A100 GPUs.
Fixed random seeds are used for seed sampling, tree expansion,
optimization, and evaluation. For every run, we retain tree-size and
valid-path statistics, distributions of $Q(s)$ and $Q(s,a)$, AFPO
component losses and subpath residuals, generated outputs, and the
complete automatic and pairwise evaluation records. These artifacts
support reconstruction of the search corpus and trace each evaluated
output to its originating dialogue prefix.

\begin{figure}[t]
    \centering
    \includegraphics[width=1.0\linewidth]{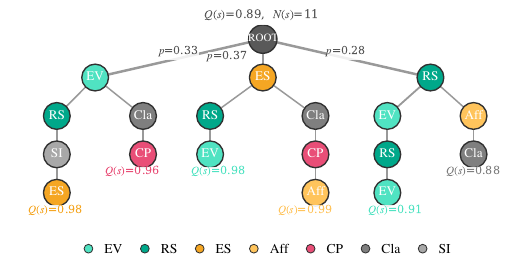}
    \caption{Representative MCTS-based affective trajectory tree. Nodes denote emotional-support strategies, edges represent strategy transitions, and edge widths reflect visitation frequencies. The annotations report policy priors \(p\), state values \(Q(s)\), and root visit count \(N(s)\).}
    \label{fig:tree-fig}
\end{figure}

\section{Human and GPT-5.2 Pairwise Evaluation Protocol}
\label{app:human_eval}

This appendix describes the pairwise evaluation protocol used in
Table~\ref{tab:combined_evaluation}(b). It specifies the evaluation
data, annotator qualifications, evaluation criteria, vote aggregation
procedure, agreement analysis, human evaluation forms, and GPT-5.2
judge prompt.

\subsection*{H.1 Evaluation Data}

Pairwise evaluation is conducted on 100 held-out multi-turn emotional
support conversation (ESC) dialogues. For each response-level
evaluation instance, two candidate Supporter responses are collected
from the systems being compared.

We also conduct trajectory-level evaluation using complete multi-turn
dialogues. This evaluation examines the development of support
strategies and emotional states over the full interaction rather than
assessing only an isolated response.

Candidate order is randomized independently for each instance. System
identities are hidden from both human annotators and the automatic
judge. All candidates are therefore evaluated without access to their
source models.

\subsection*{H.2 Annotators and Evaluation Procedure}

The human evaluation is conducted by five annotators with relevant
training or practical experience in psychological counseling. Every
evaluation instance is independently assessed by all five annotators.

The evaluation is fully blinded. System identities are concealed, and
candidate order is randomized independently for each instance. The
annotators base their judgments only on the provided background,
dialogue context, candidate responses, and complete dialogue
trajectories.

Before the formal evaluation, all annotators review the evaluation
criteria, decision rules, and annotation forms. The same instructions
and criterion definitions are used throughout the evaluation.

In addition to response-level preference, the annotators evaluate
complete multi-turn trajectories. The trajectory-level evaluation
focuses on the appropriateness of the strategy sequence, whether the
Seeker's emotional state improves over the interaction, whether advice
is introduced at an appropriate time, and whether the dialogue
contains potentially harmful content.

\subsection*{H.3 Evaluation Criteria}

\paragraph{Response-level criteria.}

Each response pair is evaluated independently on five criteria. For
each criterion, an annotator returns one ternary decision:
\texttt{1}, \texttt{2}, or \texttt{tie}.

\smallskip
\noindent\textbf{Fluency.}
Preference is given to the response that is more natural, clear, and
well-formed, without awkward phrasing, unnecessary repetition, or
conspicuous template-like language.

\smallskip
\noindent\textbf{Identification.}
Preference is given to the response that more accurately identifies
the Seeker's emotions, concerns, and underlying needs in the current
dialogue context.

\smallskip
\noindent\textbf{Comfort.}
Preference is given to the response that provides stronger emotional
support and appropriate reassurance while maintaining a warm,
non-judgmental, and context-sensitive tone.

\smallskip
\noindent\textbf{Suggest.}
Preference is given to the response that offers more useful and
feasible suggestions or next-step guidance when such guidance is
appropriate. Premature, unsafe, irrelevant, or overly prescriptive
advice is discouraged.

\smallskip
\noindent\textbf{Overall.}
Preference is given to the response that provides better emotional
support for the current dialogue turn when all preceding criteria are
considered together.

\paragraph{Trajectory-level criteria.}

Each pair of complete multi-turn trajectories is evaluated on four
additional criteria. For each criterion, an annotator returns
\texttt{1}, \texttt{2}, or \texttt{tie}.

\smallskip
\noindent\textbf{Strategy sequence.}
Preference is given to the trajectory whose support strategies follow
a more coherent and context-appropriate progression across dialogue
turns.

\smallskip
\noindent\textbf{Emotional improvement.}
Preference is given to the trajectory that better supports a gradual
improvement in the Seeker's emotional state over the interaction.

\smallskip
\noindent\textbf{Advice timing.}
Preference is given to the trajectory that introduces suggestions or
actionable guidance at a more appropriate stage, after sufficient
understanding and emotional support have been established.

\smallskip
\noindent\textbf{Potential harm.}
Preference is given to the trajectory with a lower risk of
invalidation, unsafe guidance, excessive certainty, inappropriate
pressure, premature advice, or other potentially harmful behavior.

\subsection*{H.4 Agreement, Vote Aggregation, and Confidence Intervals}

The five annotator decisions are aggregated independently for each
evaluation criterion. A final decision is determined by majority
voting. Candidate~1 is selected when it receives at least three votes,
and Candidate~2 is selected when it receives at least three votes. A
majority of \texttt{tie} votes produces a final tie.

Because each annotator chooses among three possible decisions, the
votes may occasionally follow a \(2{:}2{:}1\) split. In this case, no
decision receives a strict majority, and the instance is recorded as
a tie.

Candidate-order randomization is then reversed to recover the
identities of the compared systems. When Candidate~1 is generated by
AFlow, a final decision of \texttt{1} is recorded as an AFlow win and
a decision of \texttt{2} as an AFlow loss. When Candidate~2 is
generated by AFlow, this mapping is reversed. A final decision of
\texttt{tie} is always recorded as a tie.

For each opponent system and evaluation criterion, win, tie, and loss
percentages are calculated over all 100 dialogue instances.
Inter-annotator agreement is measured using Fleiss'~\(\kappa\),
computed separately for each response-level and trajectory-level
criterion.

Uncertainty is estimated using non-parametric bootstrap resampling
with the dialogue as the sampling unit. We report 95\% bootstrap
confidence intervals for the aggregated pairwise evaluation results.

\subsection*{H.5 Human Evaluation Forms}

\paragraph{Response-level pairwise evaluation form.}
We conduct pairwise human evaluation at both the response and
trajectory levels. Response-level evaluation assesses the quality of
an individual Supporter turn, while trajectory-level evaluation
examines strategy progression, emotional improvement, advice timing,
and potential harm across the full dialogue. Candidate identities are
hidden and their order is randomized.
\vspace{2pt}
\setlength{\fboxsep}{4pt}

\begin{figure*}[!t]
\centering
\fbox{%
\begin{minipage}{0.97\textwidth}
\footnotesize

\noindent\textbf{Task.}
Given the background and dialogue context, compare two candidate
Supporter responses and select \texttt{1}, \texttt{2}, or
\texttt{tie} for each criterion.

\vspace{4pt}\hrule\vspace{4pt}

\noindent\textbf{Background.}
Emotion: \ph{emotion\_type}\UIsep
Problem: \ph{problem\_type}\UIsep
Situation: \ph{situation}\UIsep
Memory (optional): \ph{memory}

\vspace{2pt}

\noindent\textbf{Dialogue Context.}
\ph{context\_human}

\vspace{4pt}\hrule\vspace{4pt}

\noindent
\begin{minipage}[t]{0.48\linewidth}
\textbf{Candidate 1}\\
Strategy: \ph{strg\_1}\\
Response: \ph{resp\_1}
\end{minipage}
\hfill
\begin{minipage}[t]{0.48\linewidth}
\textbf{Candidate 2}\\
Strategy: \ph{strg\_2}\\
Response: \ph{resp\_2}
\end{minipage}

\vspace{4pt}\hrule\vspace{4pt}

\noindent\textbf{Fluency.}
Which response is more fluent and natural in language?\\
\choicebtn{1}\UIsep\choicebtn{2}\UIsep\choicebtn{tie}

\vspace{2pt}

\noindent\textbf{Identification.}
Which response better identifies the Seeker's emotions and core
needs?\\
\choicebtn{1}\UIsep\choicebtn{2}\UIsep\choicebtn{tie}

\vspace{2pt}

\noindent\textbf{Comfort.}
Which response provides better comfort and emotional support?\\
\choicebtn{1}\UIsep\choicebtn{2}\UIsep\choicebtn{tie}

\vspace{2pt}

\noindent\textbf{Suggest.}
Which response offers better suggestions or next-step guidance when
appropriate?\\
\choicebtn{1}\UIsep\choicebtn{2}\UIsep\choicebtn{tie}

\vspace{2pt}

\noindent\textbf{Overall.}
Overall, which response is better for the current dialogue turn?\\
\choicebtn{1}\UIsep\choicebtn{2}\UIsep\choicebtn{tie}

\end{minipage}}
\vspace{-4pt}
\caption{Human response-level pairwise evaluation form. System
identities are hidden and candidate order is randomized.}
\vspace{-6pt}
\label{fig:human_response_form}
\end{figure*}

\paragraph{Trajectory-level pairwise evaluation form.}
GPT-5.2 is used as an automatic pairwise judge under the same five
response-level criteria as the human evaluation. It receives
anonymized candidate responses in randomized order and returns a
constrained JSON decision for each criterion, enabling consistent and
scalable evaluation.

\begin{figure*}[!t]
\centering
\fbox{%
\begin{minipage}{0.97\textwidth}
\footnotesize

\noindent\textbf{Task.}
Compare two complete emotional support dialogue trajectories and
select \texttt{1}, \texttt{2}, or \texttt{tie} for each criterion.

\vspace{4pt}\hrule\vspace{4pt}

\noindent\textbf{Background.}
Emotion: \ph{emotion\_type}\UIsep
Problem: \ph{problem\_type}\UIsep
Situation: \ph{situation}\UIsep
Memory (optional): \ph{memory}

\vspace{4pt}\hrule\vspace{4pt}

\noindent
\begin{minipage}[t]{0.48\linewidth}
\textbf{Trajectory 1}\\
\ph{complete\_trajectory\_1}
\end{minipage}
\hfill
\begin{minipage}[t]{0.48\linewidth}
\textbf{Trajectory 2}\\
\ph{complete\_trajectory\_2}
\end{minipage}

\vspace{4pt}\hrule\vspace{4pt}

\noindent\textbf{Strategy Sequence.}
Which trajectory follows a more coherent and appropriate sequence of
support strategies?\\
\choicebtn{1}\UIsep\choicebtn{2}\UIsep\choicebtn{tie}

\vspace{2pt}

\noindent\textbf{Emotional Improvement.}
Which trajectory better supports a gradual improvement in the
Seeker's emotional state?\\
\choicebtn{1}\UIsep\choicebtn{2}\UIsep\choicebtn{tie}

\vspace{2pt}

\noindent\textbf{Advice Timing.}
Which trajectory introduces suggestions at a more appropriate point
in the dialogue?\\
\choicebtn{1}\UIsep\choicebtn{2}\UIsep\choicebtn{tie}

\vspace{2pt}

\noindent\textbf{Potential Harm.}
Which trajectory has a lower risk of invalidation, unsafe guidance,
inappropriate pressure, premature advice, or other potential harm?\\
\choicebtn{1}\UIsep\choicebtn{2}\UIsep\choicebtn{tie}

\end{minipage}}
\vspace{-4pt}
\caption{Human trajectory-level pairwise evaluation form. Complete
multi-turn trajectories are anonymized and presented in randomized
order.}
\vspace{-6pt}
\label{fig:human_trajectory_form}
\end{figure*}

\subsection*{H.6 GPT-5.2 Judge Prompt}

GPT-5.2 is used as an automatic judge for response-level pairwise
evaluation under the same five response-level criteria as the human
annotators. For each instance, the model receives the dialogue
background, dialogue context, and two anonymized candidate responses.
It returns one ternary decision for each criterion.

Candidate order is randomized in the same manner as in the human
evaluation. The resulting decisions are mapped back to the
corresponding systems using the procedure described in Section~H.4.

\begin{figure*}[!t]
\centering

\begin{aflowcard}{GPT-5.2 Pairwise Judge Prompt Template}

\footnotesize
\setlength{\parindent}{0pt}
\setlength{\parskip}{0pt}

\begin{minipage}[t]{0.455\linewidth}
\vspace{0pt}

\aflowrole{System Instruction}

You are a strict and consistent evaluator for emotional support
dialogue. Follow the rubric exactly. Output \textbf{JSON only}.
Do not include explanations.

\medskip

\aflowrole{Instance Inputs}

\textbf{Background}\\
\ph{background}

\smallskip

\textbf{Dialogue Context}\\
\ph{dialogue\_context}

\smallskip

\textbf{Candidate 1}\\
\ph{candidate\_1}

\smallskip

\textbf{Candidate 2}\\
\ph{candidate\_2}

\end{minipage}
\hfill
\begin{minipage}[t]{0.515\linewidth}
\vspace{0pt}

\aflowrole{Evaluation Rubric}

\textbf{Fluency.}
Prefer the response that is more natural, clear, and well-formed in
language.

\smallskip

\textbf{Identification.}
Prefer the response that better identifies the Seeker's emotions,
concerns, and core needs.

\smallskip

\textbf{Comfort.}
Prefer the response that offers stronger emotional support and
appropriate reassurance with a fitting tone.

\smallskip

\textbf{Suggest.}
Prefer the response that provides more useful and feasible next-step
guidance when appropriate, without being premature or prescriptive.

\smallskip

\textbf{Overall.}
Prefer the response that is better on balance for the current dialogue
turn, considering all of the above criteria.

\end{minipage}

\medskip
\hrule
\medskip

\begin{minipage}[t]{0.455\linewidth}
\vspace{0pt}

\aflowrole{Output Requirement}

For each criterion, return exactly one value from
\texttt{\{1, 2, tie\}}.

The response must contain only the JSON object shown on the right.
Do not include explanations, Markdown formatting, or additional keys.

\end{minipage}
\hfill
\begin{minipage}[t]{0.515\linewidth}
\vspace{0pt}

\aflowrole{JSON Schema}

\setlength{\fboxsep}{6pt}
\fbox{%
\begin{minipage}{0.93\linewidth}
\ttfamily
\raggedright
\{\\
\hspace*{1em}"fluency": "1|2|tie",\\
\hspace*{1em}"identification": "1|2|tie",\\
\hspace*{1em}"comfort": "1|2|tie",\\
\hspace*{1em}"suggest": "1|2|tie",\\
\hspace*{1em}"overall": "1|2|tie"\\
\}
\end{minipage}%
}

\end{minipage}

\end{aflowcard}

\par\medskip
\caption{GPT-5.2 pairwise judge prompt template used for automatic
response-level evaluation.}
\label{fig:gpt52_judge_prompt}
\end{figure*}

\section{Ablation Study Details}
\label{app:depth}

This appendix provides additional details on the MCTS configuration
used for affective trajectory construction
(\S\ref{subsec:mcts}). We first describe the settings of the maximum
tree depth, rollout depth, and simulation budget. We then report the
computational costs of search trajectory construction, AFPO
optimization, and ABS inference. Finally, we describe trajectory
reuse and discuss the scalability of AFlow.

\subsection{Experimental Setup}
\label{app:sens_depth:setup}

Unless otherwise stated, all configurations follow
\S\ref{sec:method}--\S\ref{sec:afpo}. They use the same data splits,
role-separated environment, training budget, and evaluation protocol.
For each configuration, MCTS constructs affective support
trajectories. The resulting trajectories are used to optimize the
strategy policy $\pi_\theta$ and the state valuation model $V_\phi$
with the AFPO objective in Eq.~\ref{eq:afpo_Ltotal}.

We vary one MCTS hyperparameter at a time. All other parameters remain
fixed at their default values.

\paragraph{Maximum tree depth.}

The maximum tree depth is varied over
$D\in\{6,8,10,12\}$. The default value is $D=10$.

\paragraph{Rollout depth.}

The bounded rollout depth is varied over
$L\in\{1,2,3,4\}$. The default value is $L=3$. All configurations
use the same terminal-handling rule.

\paragraph{Number of simulations.}

Let $S$ denote the number of MCTS simulations performed for each root
state. We evaluate $S\in\{10,20,40\}$. The default value is $S=20$.

\paragraph{Evaluation metrics.}

Each configuration is evaluated using Strategy F1, Dist-2, PPL, and
turn-wise reward variation. Strategy F1 measures strategy alignment.
Dist-2 measures response diversity. PPL measures generation fluency.

Turn-wise reward variation is measured as the average absolute reward
difference between adjacent dialogue turns. A lower value indicates
smaller reward fluctuations and a smoother reward trajectory.

\subsection{Computational Cost and Scalability}
\label{app:computational_cost}

All experiments were conducted on two compute servers with
16 NVIDIA A100 80GB GPUs in total. We report search cost,
optimization cost, and inference search cost separately.

GPU-hours, API calls, and token counts use different units. We
therefore report these measurements separately. Inference search cost
counts only the additional calls introduced by search. It does not
include standard response generation.

\paragraph{Overall computational cost.}

Table~\ref{tab:app_overall_cost} summarizes the computational costs
of DPO, AFPO optimization, and full AFlow. A dash indicates that no
additional search cost is introduced.

\begin{table}[t]
\centering
\scriptsize
\setlength{\tabcolsep}{3.0pt}
\renewcommand{\arraystretch}{1.12}
\resizebox{\columnwidth}{!}{
\begin{tabular}{@{}lccc@{}}
\toprule
Method
& Search Cost
& Optimization Cost
& Inference Search Cost \\
\midrule
DPO
& --
& 5.67 GPU-h
& -- \\
AFlow
& \shortstack{91.5K API calls\\
123.56M input tokens\\
2.82M output tokens}
& 5.62 GPU-h
& \shortstack{ABS\\12 calls/turn} \\
\bottomrule
\end{tabular}
}
\caption{Computational cost comparison. Search cost includes API
calls and estimated token usage for trajectory construction.
Inference search cost counts the additional calls introduced by
Affective Beam Search.}
\label{tab:app_overall_cost}
\end{table}

DPO optimization requires 5.67 GPU-hours, while AFPO optimization
requires 5.62 GPU-hours. Their optimization costs are comparable.
The additional computation of AFlow mainly comes from search
trajectory construction and ABS inference.

Search trajectory construction for full AFlow requires 91,500 API
calls. It consumes approximately 123.56 million input tokens and
2.82 million output tokens. This corresponds to approximately
126.38 million tokens in total.

After the trajectories have been constructed, AFPO directly uses the
stored trajectory data for optimization. The Seeker, Supporter, and
Rewarder are not called again to reconstruct trajectories during each
optimization epoch.

At inference time, DPO and AFPO without ABS introduce no additional
search calls. Full AFlow uses Affective Beam Search, which adds
12 model calls per dialogue turn. This number represents the
inference search budget rather than the actual inference latency.
The observed latency also depends on model serving speed, generation
length, and the execution of candidate calls.

\paragraph{Role-level profiling and trajectory statistics.}

We additionally profile one search tree with $S=10$ and $L=1$. This
configuration is used only for role-level cost analysis. It is not
the default configuration used in the main experiments.

Table~\ref{tab:search_profile} combines the role-level
profiling results and the trajectory statistics. Panel~(a) reports
the calls and token usage of the Supporter, Seeker, and Rewarder.
Panel~(b) reports the size of the trajectory data used for AFPO
optimization.

\begin{table}[t]
\centering
\scriptsize
\setlength{\tabcolsep}{2.8pt}
\renewcommand{\arraystretch}{1.10}

\resizebox{\columnwidth}{!}{
\begin{tabular}{@{}lrrr|lrr@{}}
\toprule
\multicolumn{4}{c|}{\textbf{(a) Role-level profiling}}
&
\multicolumn{3}{c}{\textbf{(b) Trajectory statistics}}
\\
\cmidrule(lr){1-4}
\cmidrule(lr){5-7}

Role
& Calls
& Input Tokens
& Output Tokens
&
Statistic
& Training
& Validation
\\
\midrule

Supporter
& 20
& 14,144
& 666
&
Search trees
& 1,500
& 300
\\

Seeker
& 20
& 32,895
& 436
&
Retained trajectories
& 15,000
& 3000
\\

Rewarder
& 10
& 19,661
& 62
&
Subpath instances
& 572,400
& 24270
\\

Total
& \textbf{50}
& \textbf{66,700}
& \textbf{1,164}
&
Avg.\ trajectory length
& 8.25
& 8.09
\\

\bottomrule
\end{tabular}
}

\caption{
Role-level profiling and search-derived trajectory statistics.
Panel (a) profiles one diagnostic search tree with $S=10$
and $L=1$. Panel (b) reports the statistics of the training
corpus constructed from 1,000 ExTES seed dialogues and
500 ESConv seed dialogues after trajectory validity filtering.
Validation statistics are reported separately.
}
\label{tab:search_profile}
\end{table}

Under the diagnostic configuration, one search tree requires
50 role-model calls. The Supporter and Seeker are each called
20 times, while the Rewarder is called 10 times. Together, the three
roles consume approximately 66,700 input tokens and generate
approximately 1,164 output tokens.

The Seeker has the largest input-token cost because its input contains
both the dialogue history and the response generated by the
Supporter. The Rewarder produces shorter outputs, but its input still
contains the dialogue context and the candidate response.

The role-level statistics include only the Supporter, Seeker, and
Rewarder. They do not include the Strategy module. Panel~(a)
therefore shows the relative costs of the three simulation roles
rather than the cost of the complete trajectory-construction
pipeline.

As shown in Panel~(b), the search process produces 1,000 training
trajectories and 200 validation trajectories. These trajectories
contain 6,478 and 1,410 decision nodes, respectively. Their average
lengths are 6.48 and 7.05 decision nodes per trajectory.

\paragraph{Token estimation.}

The 91,500 API calls are counted from the execution logs. Token counts
are estimated from character-level logs because the current
implementation does not store the exact token-usage fields returned
by the API.

We use an approximate conversion of four characters per token. Under
this conversion, trajectory construction consumes approximately
123.56 million input tokens and 2.82 million output tokens. The same
conversion is applied to all token statistics in this section. These
values describe the scale of computation and are not exact tokenizer
counts.

\paragraph{Trajectory reuse.}

The 1,000 training trajectories are reused for 20 optimization
epochs. This corresponds to 20,000 trajectory-level data accesses.
The 6,478 training decision nodes are also reused in every epoch,
resulting in 129,560 decision-node-level data accesses.

These statistics describe repeated access to the same training data.
They do not represent 20,000 independent trajectories or 129,560
independent decision nodes.

MCTS is not executed again during each AFPO optimization epoch. The
search trajectories are constructed once and reused throughout all
20 epochs. Therefore, the 91,500 API calls are not repeated when the
number of optimization epochs increases.

\paragraph{Scalability.}

The computational cost of AFlow consists of three main components.
The first component is search trajectory construction. Its cost
depends on the number of search trees, the simulation budget $S$, the
rollout depth $L$, and the maximum tree depth $D$.

The second component is AFPO optimization. It operates on stored
trajectory data and requires 5.62 GPU-hours in our experiments.

The third component is ABS inference. ABS introduces 12 additional
search calls per dialogue turn. Its total search cost therefore grows
with the number of dialogue turns.

Different search trees have no computational dependencies. Their
construction can be distributed across multiple GPUs or servers.
This parallelism can reduce wall-clock time, but it does not reduce
the total number of API calls or tokens.

Candidate branches within ABS can also be processed concurrently.
This can reduce inference latency, but it does not change the search
budget of 12 calls per turn.

Overall, the additional computation of AFlow is concentrated in
search trajectory construction and ABS inference. The constructed
trajectories are reused across multiple AFPO optimization epochs, so
MCTS does not need to be repeated during each epoch.

\subsection{Recommended Configuration}
\label{app:sens_depth:recommend}

The main experiments use $D=10$, $L=3$, and $S=20$. This
configuration provides sufficient search depth, rollout range, and
simulation budget for constructing multi-turn strategy trajectories.
It also keeps API usage and token consumption within the available
computational budget.

The role-level profiling experiment uses $S=10$ and $L=1$. This
diagnostic configuration is used only to measure the costs of the
Supporter, Seeker, and Rewarder. It does not replace the default
configuration used in the main experiments.

The complete cost profile of AFlow includes 91,500 API calls for
search trajectory construction, approximately 123.56 million input
tokens, 2.82 million output tokens, 5.62 GPU-hours for AFPO
optimization, and 12 additional search calls per dialogue turn for
ABS inference. These costs are reported separately to distinguish
trajectory construction, model optimization, and inference search.

\section{Case Study}
\label{app:case}

This appendix provides qualitative evidence of how AFlow adapts its
support strategy over multi-turn interactions. Rather than treating a
case as proof of general performance, we use each example to inspect a
specific decision point: whether the model should continue exploring,
validate the Seeker's experience, or move toward action-oriented
support. For each case, we report a dialogue excerpt, the complete
four-turn strategy and reward trajectories, and a comparison with a
strong baseline. The excerpt shows the first three supporter turns;
the fourth turn is omitted for space but remains included in the
trajectory summary. All examples are anonymized.

The analysis focuses on two properties. First, \emph{strategy
progression} concerns whether the selected support strategy changes in
response to the evolving dialogue state rather than remaining fixed
across turns. Second, \emph{affective-flow consistency} concerns
whether these transitions are accompanied by a stable improvement in
the reward trajectory. Together, the two views help distinguish a
locally fluent response from a sequence of responses that remains
coherent over the full interaction.

\subsection{Case Selection Criteria}

The cases are selected to expose long-horizon decisions under
different emotional conditions. They are not chosen only because
AFlow performs well; instead, each example contains a concrete
transition at which a superficially reasonable strategy can become
misaligned if it is introduced too early or repeated for too long.

\smallskip
\noindent\textbf{(1) Context diversity.}
The three examples cover anticipatory anxiety, interpersonal conflict,
and achievement-related self-doubt. These settings differ in both the
source of distress and the type of support that may become appropriate
later in the dialogue. This variation reduces the risk of drawing
qualitative conclusions from a single emotional setting.

\smallskip
\noindent\textbf{(2) Comparable interaction difficulty.}
The selected dialogues have similar turn budgets and contain
non-trivial constraints, including uncertainty about the Seeker's
immediate need, repeated negative affect, and reluctance to take
action. In each case, a generic supportive response would be plausible
at the beginning, but later turns require more precise adaptation.

\smallskip
\noindent\textbf{(3) Informative strategy transitions.}
Each case contains at least one turn at which strategy timing is
non-obvious. These turns reveal whether the model responds to the
evolving dialogue state or falls into premature advice and repetitive
comfort. They are therefore especially useful for examining the effect
of prefix-level supervision on long-horizon decision making.

\subsection{Representative Cases}

The following cases are organized around the main transition that
distinguishes AFlow from the baseline. The dialogue excerpt illustrates
the local context in which the decision is made, while the strategy
and reward summary shows how the choice affects the remainder of the
interaction.

\paragraph{Case 1: Anxiety under uncertainty.}
The Seeker is preparing for another job interview after several
rejections and expects to fail again. The central difficulty is not
general preparation but panic during technical questions. A useful
supporter must therefore identify the specific source of anxiety before
moving to coping strategies or preparation advice.

\smallskip
\begin{aflowcard}{Dialogue Excerpt}

\aflowrole{Seeker 1}
I have another interview tomorrow, but I'm terrified. I feel like
I'm just going to fail again.

\smallskip
\aflowrole{AFlow 1}
I hear how stressful this is for you. What specifically about
tomorrow's interview is making you feel most anxious?

\aflowsep
\aflowrole{Seeker 2}
It's the technical part. I always freeze up when they ask coding
questions.

\smallskip
\aflowrole{AFlow 2}
It makes sense that past experiences would make you nervous about
the technical round. That's a lot of pressure to carry.

\aflowsep
\aflowrole{Seeker 3}
Yeah, exactly. I just don't know how to stop panicking in the
moment.

\smallskip
\aflowrole{AFlow 3}
Since the panic is the main hurdle, maybe we could try a quick
grounding technique right before you start coding? Or would you
prefer to review some basic concepts to feel more prepared?

\end{aflowcard}

\medskip
\begin{aflowcard}{Turn-wise Strategy and Reward Summary}

\aflowrole{AFlow Strategies}
Clarification $\rightarrow$ Emotional Validation
$\rightarrow$ Collaborative Planning $\rightarrow$ Feedback

\smallskip
\aflowrole{Baseline Strategies}
Affirmation $\rightarrow$ Suggest Options
$\rightarrow$ Suggest Options $\rightarrow$ Affirmation

\aflowsep
\aflowrole{AFlow Rewards}
$0.65 \rightarrow 7.2 \rightarrow 0.81 \rightarrow 0.85$

\smallskip
\aflowrole{Baseline Rewards}
$0.60 \rightarrow 0.58 \rightarrow 0.55 \rightarrow 0.62$

\end{aflowcard}

\smallskip
\noindent\textbf{Analysis.}
The main difference appears at the second supporter turn. The
baseline moves to options before fully addressing why the technical
round feels threatening. Although offering practice advice is not
inherently inappropriate, it does not yet respond to the Seeker's
fear of freezing under pressure. AFlow instead validates the effect
of the earlier failures and waits until the obstacle is sufficiently
specified before proposing two low-pressure next steps.

This ordering has two consequences. First, the Seeker's concern is
narrowed from a global expectation of failure to a concrete difficulty
that can be discussed. Second, the eventual transition to collaborative
planning is grounded in information elicited during the dialogue,
rather than introduced as a generic recommendation. The AFlow reward
therefore rises throughout the interaction, whereas the baseline
declines after its premature shift to suggestions. The case illustrates
how clarification and validation can function as preparatory stages
for action rather than as independent, interchangeable responses.

\medskip
\paragraph{Case 2: Interpersonal conflict and readiness to act.}
The Seeker is frustrated by a roommate who repeatedly leaves the
kitchen untidy but is reluctant to initiate a confrontation. The main
decision is when to move from acknowledgment to practical guidance.
Here, advice that is reasonable in content may still be poorly timed
if it ignores the Seeker's concern about damaging the relationship.

\smallskip
\begin{aflowcard}{Dialogue Excerpt}

\aflowrole{Seeker 1}
My roommate left the kitchen a mess again. I'm so angry, but I hate
fighting.

\smallskip
\aflowrole{AFlow 1}
It sounds incredibly frustrating to come home to a mess, especially
when you're trying to avoid conflict.

\aflowsep
\aflowrole{Seeker 2}
Exactly! I don't want to be the ``nagging'' one, but I can't live
like this.

\smallskip
\aflowrole{AFlow 2}
It's a tough position to be in---wanting a clean space while also
wanting peace. You value harmony in the house.

\aflowsep
\aflowrole{Seeker 3}
I do. But maybe I need to say something this time.

\smallskip
\aflowrole{AFlow 3}
If you decide to speak up, we could brainstorm a way to say it that
feels firm but not aggressive. How does that sound?

\end{aflowcard}

\medskip
\begin{aflowcard}{Turn-wise Strategy and Reward Summary}

\aflowrole{AFlow Strategies}
Restatement $\rightarrow$ Affirmation
$\rightarrow$ Collaborative Planning $\rightarrow$ Provide Information

\smallskip
\aflowrole{Baseline Strategies}
Suggest Options $\rightarrow$ Suggest Options
$\rightarrow$ Restatement $\rightarrow$ Affirmation

\aflowsep
\aflowrole{AFlow Rewards}
$0.68 \rightarrow 0.75 \rightarrow 0.83 \rightarrow 0.86$

\smallskip
\aflowrole{Baseline Rewards}
$0.55 \rightarrow 0.52 \rightarrow 0.60 \rightarrow 0.65$

\end{aflowcard}

\smallskip
\noindent\textbf{Analysis.}
This case illustrates the importance of pacing. AFlow first reflects
the conflict and then recognizes the Seeker's wish to preserve
harmony. These two responses do more than provide comfort: they make
the Seeker's competing goals explicit. Once the Seeker states a
willingness to speak up, collaborative planning becomes aligned with
their expressed readiness.

The baseline instead offers solutions before addressing the fear of
confrontation, which is followed by an early reduction in reward.
Its later return to restatement and affirmation partially repairs the
interaction, but the sequence is less coherent because supportive
acknowledgment arrives after the advice. The advantage therefore lies
not simply in giving more advice, but in ordering validation and action
so that each strategy prepares the dialogue for the next one.

\paragraph{Case 3: Self-doubt and actionable reframing.}
The Seeker feels overwhelmed by a thesis deadline, compares themself
with peers, and interprets difficulty making progress as evidence of
low ability. The key challenge is to avoid remaining in generic
reassurance while still preserving emotional support. A useful
transition must acknowledge the self-doubt while opening a more
actionable interpretation of the difficulty.

\medskip

\noindent\textbf{Analysis.}
The baseline remains in a repetitive affirmation pattern and does
not substantially update its strategy after the Seeker describes
being unable to write. Reassurance may reduce distress momentarily,
but repeated affirmation does not explain why progress has stalled
or create a manageable next step. AFlow first acknowledges the
emotional burden, then reframes the difficulty as possible exhaustion,
and finally offers a small, reversible action.

\smallskip
\begin{aflowcard}{Dialogue Excerpt}

\aflowrole{Seeker 1}
I'm never going to finish this thesis. Everyone else is doing so
well, and I'm stuck.

\smallskip
\aflowrole{AFlow 1}
It's easy to compare yourself with others when you're stuck. That
feeling of being ``behind'' is really heavy.

\aflowsep
\aflowrole{Seeker 2}
Yeah, I just feel stupid. I've been staring at the same page for
hours.

\smallskip
\aflowrole{AFlow 2}
Staring at a blank page doesn't mean you aren't capable; it often
means you're burnt out. Have you taken a break today?

\aflowsep
\aflowrole{Seeker 3}
No, I felt guilty stopping. But maybe I am just tired.

\smallskip
\aflowrole{AFlow 3}
Guilt can be a tricky trap. How about we try a small 15-minute break
just to reset your brain?

\end{aflowcard}

\medskip
\begin{aflowcard}{Turn-wise Strategy and Reward Summary}

\aflowrole{AFlow Strategies}
Empathetic Statement $\rightarrow$ Reframe/Clarification
$\rightarrow$ Suggest Options $\rightarrow$ Affirmation

\smallskip
\aflowrole{Baseline Strategies}
Affirmation $\rightarrow$ Affirmation
$\rightarrow$ Affirmation $\rightarrow$ Suggest Options

\aflowsep
\aflowrole{AFlow Rewards}
$0.65 \rightarrow 0.74 \rightarrow 0.82 \rightarrow 0.85$

\smallskip
\aflowrole{Baseline Rewards}
$0.65 \rightarrow 0.60 \rightarrow 0.58 \rightarrow 0.62$

\end{aflowcard}

\smallskip

The reframing is important because it changes the implied problem from
a stable lack of ability to a temporary state that can be addressed.
The subsequent suggestion is therefore less likely to feel dismissive
or overly demanding. The steadily increasing reward is consistent with
AFlow moving away from repetitive comfort toward context-sensitive
guidance while preserving emotional continuity.

\subsection{Cross-Case Discussion}

Across the three examples, AFlow does not apply a fixed
exploration--comfort--action template. The transition point varies
with the dialogue state. In Case 1, action follows clarification and
validation of a specific fear. In Case 2, action is delayed until the
Seeker expresses readiness. In Case 3, action follows a reframing that
changes how the difficulty is understood. The model therefore uses
similar high-level strategy families in different orders depending on
the local evidence available in the dialogue.

The reward trajectories also reveal a consistent distinction between
response quality and sequence quality. Several baseline responses are
individually plausible, but their ordering produces premature advice
or repeated affirmation. By contrast, AFlow's advantage emerges from
how strategies compose across turns. This observation is aligned with
the motivation of AFPO: downstream support outcomes provide a signal
for improving earlier strategy decisions rather than evaluating each
response in isolation.